\documentclass{article}

\usepackage{amsmath, amssymb, natbib, graphicx, url, algorithm2e, amsthm}
\usepackage{geometry}
\bibpunct{(}{)}{;}{a}{,}{,}
\usepackage{hyperref}
\hypersetup{colorlinks,
            linkcolor=blue,
            citecolor=blue,
            urlcolor=magenta,
            linktocpage,
            plainpages=false}

\usepackage{thmtools,thm-restate}
\usepackage{{James}}

\title{A Mathematical Theory of Attention}

\usepackage{xcolor}
\usepackage[normalem]{ulem}
\usepackage{mathrsfs}

\usepackage[titles]{tocloft}

\newcommand{\attn}{\mathrm{Attention}}
\newcommand{\ffn}{\mathrm{FFN}}

\newcommand{\multiselfattn}{\mathrm{MultiHeadSelfAttention}}
\newcommand{\selfattn}{\mathrm{SelfAttention}}
\newcommand{\softmax}{\mathrm{softmax}}
\newcommand{\softmatch}{\mathrm{softmatch}}
\newcommand{\transf}{\mathrm{Transformer}}

\newcommand{\bdot}{\bullet}

\newcommand{\E}{\mathds{E}}

\newcommand{\K}{\mathds{K}}

\newcommand{\V}{\mathds{V}}

\newcommand{\W}{\mathds{W}}
\newcommand{\KL}{\mathds{KL}}

\newcommand{\diam}{\mathrm{diam}}

\newcommand{\defn}[1]{\textbf{#1}}

\newtheorem*{assumption*}{Assumption}
\newtheorem{assumption}{Assumption}

\newtheorem{theorem}{Theorem}
\newtheorem{lemma}[theorem]{Lemma}
\newtheorem{proposition}[theorem]{Proposition}
\newtheorem{remark}[theorem]{Remark}
\newtheorem{corollary}[theorem]{Corollary}
\newtheorem{definition}[theorem]{Definition}

\newcommand{\nnewline}{\newline}

\author{
	James Vuckovic\footnote{Microsoft, \url{v-javuc@microsoft.com}}\\
 	Aristide Baratin\footnote{Mila \& Universit\'e de Montr\'eal,  \url{aristide.baratin@umontreal.ca}}\\
 	Remi Tachet des Combes\footnote{Microsoft Research Montreal, \url{retachet@microsoft.com}}
}

\begin{document}

\maketitle

\begin{abstract}%
    Attention is a powerful component of modern neural networks across a wide variety of domains. However, despite its ubiquity in machine learning, there is a gap in our understanding of attention from a theoretical point of view. We propose a framework to fill this gap by building a mathematically equivalent model of attention using measure theory. With this model, we are able to interpret self-attention as a system of self-interacting particles, we shed light on self-attention from a maximum entropy perspective, and we show that attention is actually Lipschitz-continuous (with an appropriate metric) under suitable assumptions. We then apply these insights to the problem of mis-specified input data; infinitely-deep, weight-sharing self-attention networks; and more general Lipschitz estimates for a specific type of attention studied in concurrent work.
\end{abstract}



\tableofcontents

\section{Introduction}

Attention~\citep{bahdanau2014neural,vaswani2017attention} has recently joined the multi-layer perceptron, convolution, and recurrent neural network cell as a fundamental building block of modern neural networks. Mathematical theories of the multilayer perceptron \citep{cybenko1989approximation, hornik1989multilayer} and convolution \citep{mallat2016understanding, cohen2016expressive} exist, and explain these operations as instances of broader mathematical objects. From these explanations, we gain the support of these theories to further investigate the behaviour of neural networks containing these constituents.\nnewline

However, despite its meteoric rise within deep learning (or perhaps, because of it), we believe there is a gap in our theoretical understanding of what attention is, and why it works. In particular, attention cannot be considered an instance of any existing building blocks (e.g. linear transformation, convolution), therefore we have mainly intuitive explanations and empirical evidence -- see the recent rise of ``Bertology''~\citep{tenney2019bert,clark2019does} -- to understand its inner workings. Additionally, there is recent evidence that the success of attention is not restricted to the natural language~\citep{velivckovic2017graph,parmar2018image,baker2019emergent} and is a general phenomenon.

\paragraph{Objective and Contributions.} Our goal is to study attention in a framework that allows for a systematic and principled exploration of its properties. We approach this goal in three steps: $(i)$ Identify a natural mathematical model with which to study attention; $(ii)$  Leverage this model to establish basic results about the properties of attention; $(iii)$ Apply these insights to deepen our understanding and intuition about attention-based models. We provide a framework, built using measure theory and Markov kernels, as an alternative (but equivalent) modelling paradigm which grants us access to the powerful tools of analysis on the space of measures. One example of these tools at work, to be proven in Section~\ref{sec:stability}, is a proof of the following continuity estimate
\[
    \|\attn(q_1,K,V) - \attn(q_2,K,V)\| \leq \text{constant} \cdot \|q_1 - q_2\|
\]and an upper bound on what ``constant'' should be.

\medskip

\noindent Specifically, our contributions are:

\begin{enumerate}
	\item We show that attention can be interpreted as system of interacting particles, and we model the evolution of this system as a nonlinear transformation of probability measures;
	\item We show that attention contains the solution to a maximum entropy problem and a minimum entropy projection, and is therefore characterized by these problems;
	\item We show that attention is a Lipschitz-continuous operation under suitable assumptions, and give quantitative estimates of the Lipschitz coefficient;
	\item We apply these theories to study sub-sampled input data; infinitely-deep, weight-sharing self-attention networks; and relaxing the boundedness assumptions in our Lipschitz estimates to obtain more general results than other concurrent works.
\end{enumerate}

The paper is organized as follows. We first introduce preliminaries in Section~\ref{sec:prelim}. In Section~\ref{sec:attention}, we describe attention and the Transformer using measure theory. We then qualitatively study how attention works using Kullback-Leibler projections (Section~\ref{sec:analysis}), and in Section~\ref{sec:stability} we obtain quantitative Lipschitz continuity estimates for self-attention. We apply these results to some concrete problems in Section~\ref{sec:applications} and review related work in Section~\ref{sec:related}.

\section{Preliminaries}\label{sec:prelim}

\subsection{Attention and Structured Data}\label{sub:structured}

The fundamental definition of attention is due to \cite{bahdanau2014neural}, which we provide below with some additional terminology for the various components that we will study later on.
\begin{definition}[Attention, \cite{bahdanau2014neural}]\label{defn:attention}
		Let $\calQ\subseteq\R^{d_q}$, $\calK\subseteq\R^{d_k}$, $\calV\subseteq \R^{d_v}$ the query-space, key-space, and value-space respectively. Let $K=\{k_1,\dots,k_N\}\subset\calK$ be a set of keys, $V=\{v_1,\dots,v_N\}\subset\calV$ a set of values, and $q\in\calQ$ a query. Finally, let $a:\calQ\times\calV\to \R$ be a 
		similarity function. Attention is the mapping
		\vspace{-0.3cm}
		\[
			\attn(q,K,V) := \sum_{i=1}^N \softmatch_a(q,K)_i\cdot v_i,
		\]
		\vspace{-0.1cm}
		where $\softmatch_a(q,V)$ is a probability distribution over the elements of $K$ defined as
        \begin{align}\label{eq:softmatch}
			\softmatch_a(q,K)_i:= \frac{\exp(a(q,k_i))}{\sum_{j=1}^N \exp(a(q,k_j))}=\softmax_j(\{a(q,k_j)\}_j).
        \end{align}
	\end{definition}
	
	While $\attn(\bdot,K,V)$ is defined point-wise for a given query, it is almost always used to process a set of queries $Q=\{q_1,\dots,q_M\}\subset\calQ$ in parallel. Thus, we will usually write $\attn(Q,K,V) := \{\attn(q_i,K,V)\}_{i=1}^M$. Also, while we must have \mbox{$|K|=|V|=N$}, $M$ does not have to equal $N$ in general. Moreover, when $K=V=Q$, we call the following mapping \textbf{self-attention}:
	\[
	    Q\mapsto \selfattn(Q) := \attn(Q,Q,Q).
	\]

	In the sequel, we are primarily interested in the phenomenon of self-attention as it \emph{can be composed to arbitrary depth}, making it a key building block of many neural network architectures.\nnewline

The attention operation is very versatile, and has become a key component of many neural networks that operate on discrete data with some type of structure on the elements, i.e. ``structured data''. 
More precisely, consider a finite collection of elements $X=\{x_1,\dots,x_T\}$ indexed by a finite set $t\in \calT$, where each $x_t\in \calV$ for some finite \emph{vocabulary} $\calV$.  
In the representation learning framework of attention, we assume $\calV$ has been embedded in $\R^d$, so we identify each $x_t$ with one of \emph{finitely-many} vectors in $\R^d$. Hence, we will consider $X\subset E$, where $E$ is some fixed subset of $\R^d$. We assume that the structure (positional information, adjacency information, etc) is encoded in these vectors.

\subsection{Markov Kernels}
\label{sub:features}

In the sequel, $(E,\calE)$ denotes a subset of $\R^d$ endowed with its Borel $\sigma$-algebra, and  $\calP(E)$ the space of all probability measures on $E$. We use the following notation for expectations with respect to $\mu \in \calP(E)$: for a real-valued measurable function $f$, we denote  $\mu(f):= \int f(x)\mu(\dd x)$ when it exists. For a vector-valued function $F : E \to \R^l$ with $F(x) = [F_1(x), \cdots,  F_l(x)]^T$, we denote $\mu(F) := [\mu(F_1), \cdots, \mu(F_l)]^T$ when all the components exist.

\medskip

Our framework will heavily rely on linear transformations of measures modelled by {\bf Markov kernels}; see e.g.~\cite{moral2004feynman}.

\begin{definition}[Markov kernel] A Markov kernel is an $E$-indexed family of probability measures $M(x,\dd y) \in \calP(E)$ such that 
$\forall A\in \calE, x\mapsto M(x,A)$ is measurable.  
\end{definition} 
A Markov kernel $M$ defines a linear operator $\calP(E)\to \calP(E)$ by $\mu M(\dd y):= \int \mu(\dd x) M(x,\dd y)$. It also defines a linear operator on measurable functions by $M(f)(x) := \int f(y) M(x,\dd y)$. Markov kernels $M,N$ can be composed by integration, 
$MN(x,\dd z) := \int M(x,\dd y)N(y,\dd z)$. 

\section{Modelling Attention}\label{sec:attention}
In this section, we model attention~\citep{bahdanau2014neural} and the Transformer~\citep{vaswani2017attention} in measure-theoretic language. Our construction casts the action of attention and the Transformer on a collection of vectors as a nonlinear Markov transport on $\calP(E)$ by  reformulating existing linear algebra and point-wise operations in-terms of operators on $\calP(E)$. 

\subsection{Basic Model of Attention}
The fundamental parts of $\attn$ from Definition~\ref{defn:attention} are the $\softmatch_a$ operation, the key-value correspondence, and the value-averaging w.r.t. the softmatch distribution. We will treat each of these in turn. 

\paragraph{Softmatch and Botzmann-Gibbs Transformations.} At the core of the softmatch function, and indeed attention in general, are the interactions between queries and keys. These interactions are a specific case of a nonlinear measure transformation, the Boltzman-Gibbs transformation.

\begin{definition}[Boltzmann-Gibbs Transformation]
	Let $g:E\to \R_{>0}$ be a bounded measurable function.
	The \textbf{Boltzmann-Gibbs transformation} associated to
	$g$ is the mapping $\Psi_g:\calP(E)\to \calP(E)$ defined as \[ \Psi_{g}(\nu)(dx) := \frac{g(x) \nu(dx)}{\nu(g)}.\]
\end{definition}
 
\noindent To implement the $\softmatch_a$ operation, we will need a function $G:E\times E\to \R_+^*$ taking the form $G(x,y)=\exp(a(x,y))$, where $a$ is a similarity function as in Definition~\ref{defn:attention}. We call $G$ an {\bf interaction potential}.

\begin{definition}[Softmach Kernel] For an interaction potential $G$, we call the \textbf{softmatch kernel} the family of Markov kernels  $\{\Psi_{G}(\nu)\}_{\nu\in \calP(E)}$ indexed by $\nu\in \calP(E)$, where for $A\in \calE$
\[
\Psi_{G}(\nu)(x, A) = \int_A \Psi_{G(x,\bdot)}(\nu)(\dd y).
\]
\end{definition}

To see how $\Psi_G$ can be used to model the softmatch operation, we need to introduce some simple but useful constructions from measure theory. Denote by $\calP_\delta(E):=\{\delta_x \st x\in E\}$ the subset of \defn{Dirac measures} in $\calP(E)$. There is a natural bijection between $E$ and $\calP_\delta(E)$ defined by $x\leftrightarrow \delta_x$ which will be the primary entry point for measure theory in our model of attention. We can associate to any realization of the structured data $X=\{x_1,\dots,x_T\}$ a measure in $\calP(E)$ via the \defn{empirical measure mapping}:
\[
	X\mapsto m(X):=\frac{1}{|\calT|}\sum_{t\in \calT} \delta_{x_t}.
\]In what follows, we will often use $X$ and $ \{\delta_{x_1},\dots,\delta_{x_T}\}$ interchangeably to represent the individual vectors and $m(X)$ to represent the joint configuration of $X$. We will see below that $m(X)$ is a very natural object to represent this joint configuration and to model how it behaves with attention.\nnewline

Hence consider a ``query'' representation $\delta_q$ and ``key'' representations $K=\{\delta_{k_1},\dots,\delta_{k_N}\}$ and the empirical measure $m(K)$. The softmatch kernel models the interaction between $q$ and $K$ using the left-action of the Markov kernels $\Psi_G(m(K))$ on the Dirac measure $\delta_q$ induced by integration:
\begin{align*}
	\delta_q\Psi_{G}(m(K)) &= \int \delta_{q}(\dd q')\Psi_{G(q',\bdot)}(m(K)) =\Psi_{G(q,\bdot)}(m(K)) = \sum_{s=1}^N \frac{G(q,k_s)}{\sum_{r=1}^N G(q,k_r)}\delta_{k_s}.
\end{align*}
Furthermore, given set of queries  $Q=\{\delta_{q_1},\dots,\delta_{q_M}\}$, we can leverage the linearity of integration to model the interaction between the two sets of representations $Q$ and $K$ using the same principle:
\begin{align*}
	m(Q)\Psi_{G}(m(K)) &= \frac{1}{M} \sum_{t=1}^M \int \delta_{q_t}(\dd q)\Psi_{G(q,\bdot)}(m(K)) \\
	&=\frac{1}{M} \sum_{t=1}^M \Psi_{G(q_t,\bdot)}(m(K)) = \frac{1}{M}\sum_{t=1}^M\sum_{s=1}^N \frac{G(q_t,k_s)}{\sum_{r=1}^N G(q_t,k_r)}\delta_{k_s}.
\end{align*}
This new measure represents the joint configuration of the set of queries $Q$ \emph{after} they have been allowed to interact with the keys $K$ through the potential $G$ and associated Boltzmann-Gibbs transformation. It is a \emph{weighted sum} of particle measures, and will allow us to model the softmatch from Eq.~\ref{eq:softmatch}. 

\paragraph{Key-Value Relationships.} To generalize the relationship between keys and values, we now introduce the lookup kernel.
\begin{definition}[Lookup Kernel]
	Assume that the key and value spaces are measurable spaces $(\calK,\K)$ and $(\calV,\V)$ respectively. A lookup kernel is a Markov kernel, $L:\calQ\times\V\to [0,1]$, also denoted $L(k,\dd v)$, that maps keys to distributions on values. When the mapping is a deterministic lookup table, we have $L(k,\dd v)=\sum_{i=1}^N\ind{k=k_i}\delta_{v_i}(\dd v)$.
\end{definition}

\paragraph{Averaging and Moment Projections.}
The final element of our construction is the averaging w.r.t. the set of values being attended over. To define this operation in our framework, we need to introduce the concept of moments encoded in a measure and a family of measures parameterized by a moment vector.
\begin{definition}[Moment Encoding]
    Let 
    $F:E\to E'\subset \R^{l}$ be a measurable map 
    representing an $l$-dimensional \textbf{feature map}.
    We say that a measure $\mu\in \calP(E)$ {\bf encodes a moment vector} $f\in E'$ with respect to $F$ if $\mu(F)=f$.  
\end{definition}

This notion is particularly useful when $E=E'$ and $F(x)=x$ in which case $\delta_x(F_j)=x_j$, so the Dirac measure $\delta_x$ encodes the moment vector  $x$. This will allow us to translate notions about vector-space representations to notions about measures and vice versa. 

\begin{definition}[Moment Subspace of $\calP(E)$]
    Let $F:E\to E'$ be a fixed measurable map, and suppose we have identified an \emph{injective} mapping $f\mapsto \nu_f\in \calP(E)$ such that $\nu_f$ encodes the moment vector $f$ w.r.t. $F$. Then we denote by $\calF_F=\{\nu_f\st f\in E'\}$ the \textbf{moment subspace} of all such distributions, usually omitting the subscript $F$ when there is no chance of confusion.
\end{definition}
Since we are interested in analyzing finite-dimensional representations in the (uncountable) infinite-dimensional space $\calP(E)$, a moment subspace identifies a finite-dimensional parametric family of measures, where our representations can easily be recovered. Two particular cases of moment subspaces are used in the sequel: (1) a natural exponential family with $l$-dimensional sufficient statistics $F$, and (2) the set $\calP_\delta(E) = \{\delta_x\st x\in E\}$ of Dirac measures,  with $E'=E$ and $F(x)=x$. Finally, we define the \textbf{moment projection} onto $\calF$, denoted $\Pi_\calF:\calP(E)\to \calF$, as 
\begin{equation} \label{eq:pi}
	\Pi_\calF(\mu):= \nu_{\mu(F)}.
\end{equation}
$\Pi_\calF(\mu)$ is the unique measure in $\calF$ that encodes the moments $f := \mu(F)$.\nnewline

Now we claim (to be justified in a moment) that the averaging w.r.t. input values is accomplished by the moment projection $\Pi := \Pi_\calF$ described in Eq.~\ref{eq:pi}, with $\calF=\calP_\delta(E)$ and $F(x)=x$. 

\paragraph{The Attention Kernel.}
Putting all of these elements together, we obtain our primary model for attention, the attention kernel.
\begin{definition}[Attention Kernel]\label{defn:attention_kernel}
	The attention kernel $\mbf{A}$ is the composition of the projection $\Pi$, the softmatch kernel and the lookup kernel, defined for $x \in E$ and $\mu \in \calP(E)$ as:
	\begin{align*}
        \mbf{A}_\mu(x,dz):=\Pi[\Psi_{G(x,\bdot)}(\mu)L](dz) = \Pi\left[\int \Psi_{G(x,\bdot)}(\mu)(\dd y)L(y,\bdot)\right](\dd z).
	\end{align*}
\end{definition}

This attention kernel is consistent with attention from Definition~\ref{defn:attention}, for suitable choices of $G$ and $L$.
\begin{proposition}\label{prop:measure_attn}
	Let $G(x,y)=\exp(a(x,y))$, $L(k,\dd v)=\sum_{i=N}\ind{k=k_i}\delta_{v_i}(\dd v)$, and $Q,K,V$ be as in the definition of attention. Then, using the left action of kernels on measures, the mapping:
	\[
		(Q,K,V)\mapsto \left\{\delta_{q_1}\mbf{A}_{m(K)},\dots,\delta_{q_T}\mbf{A}_{m(K)}\right\}
	\]implements attention as in Definition~\ref{defn:attention}.
\end{proposition}
\begin{proof}
	Using the remarks from earlier, for $q\in \calQ$, we have:
	\[
		[\Psi_{G(q,\bdot)}(m(K))L](\dd v) =\sum_{j=1}^N \int \frac{G(q,k_j)}{\sum_{p=1}^NG(q,k_p)}\delta_{k_j}(\dd k)L(k,\dd v) = \sum_{j=1}^N \frac{G(q,k_j)}{\sum_{p=1}^NG(q,k_p)}\delta_{v_j}(\dd v).
	\]Applying $\Pi$ thus yields: $\mbf{A}_{m(K)}(q,\dd v) = \delta_{\sum_{j=1}^N \frac{G(q,k_j)}{\sum_{p=1}^NG(q,k_p)}v_j}(\dd v)$. Using the (linear) left-action of this kernel on $\delta_{q_t}$, we then obtain:
	\[
		\delta_{q_t}\mbf{A}_{m(K)}(\dd v) = \int \delta_{q_t}(\dd q)\mbf{A}_{m(K)}(q,\dd v) = \delta_{\sum_{j=1}^N \frac{G(q_t,k_j)}{\sum_{p=1}^NG(q_t,k_p)}v_j}(\dd v).
	\]Plugging in the definition of $G$ and using the usual bijection $\delta_x\leftrightarrow x$ concludes the proof.
	\end{proof}
	
	\paragraph{Attention as a System of Interacting Particles.}
	Let us step back and understand the attention kernel $\mbf{A}$ from a higher level. Consider self-attention: we have effectively factorized the original, linear-algebraic self-attention operation into a series of measure transformations: 
	\[
		E \overset{x\mapsto \delta_x}{\longrightarrow} \calP_\delta(E)\overset{\Psi_GL}{\longrightarrow} \calP(E) \overset{\Pi}{\longrightarrow} \calP_\delta(E) \overset{\delta_x\mapsto x}{\longrightarrow} E.
	\]More importantly, we \emph{have a closed-form expression for the evolution of the joint configuration $m(Q)$ of $Q$}, i.e. $m(Q)\mapsto m(Q)\mbf{A}_{m(Q)}$. Since interaction with the joint configuration is central to the function of attention, 
	having a framework that describes the evolution of $m(Q)$ will be vital to further analysis. \nnewline

	Moreover, as we noted earlier, self-attention can be composed arbitrarily. Indeed, let $Q^0:=Q$ and consider the evolution of a the set of ``particles'' $Q^h=\{\delta_{q^h_1},\dots,\delta_{q^h_M}\}$ for $h=0,1,2,\dots,H-1$ whose dynamics are given by
	\[
	    q^{h+1}_i \sim \mbf{A}^h_{m(Q^h)}(q^h_i,\bdot)
	\]or equivalently as a measure-valued flow
	\[
	     \delta_{q^{h+1}_i} = \delta_{q^h_i}\mbf{A}^h_{m(Q^h)}.
	\]

	Our framework shows that self-attention networks, including the Transformer architecture \citep{vaswani2017attention} which is now ubiquitous in natural language processing (see below), are actually simulating deterministic interacting particle systems for a finite number of time steps corresponding to the number of layers $H$. The representations one obtains are the states of the system after $H$ steps of the dynamics. 
	We use the measure-valued flow interpretation rather than the particle interpretation since it allows us to study the analytical properties of $\mbf{A}$ to derive insights about attention. \nnewline

	\begin{remark}\label{rem:interacting}
	    Interestingly, the particle interpretation above is studied in \cite{lu2019understanding} using tools from dynamical systems theory. In particular, the authors recognize the Transformer (with the residual connection) as a coupled system of particles evolving under diffusion-convection ODE dynamics, and study this system using the a numerical scheme for the underlying ODE.
	\end{remark}

    \begin{remark}
        Let us also point out  a connection with Bayesian statistics: the case when $G(q,\bdot)=p(q|\bdot)$ is a likelihood function, $\nu\mapsto \Psi_{G(q,\bdot)}(\nu)$ is the mapping which takes a prior distribution $\nu(\dd k)$ over keys and returns a posterior distribution $P(\dd k|q)$. Moreover, under a Markov assumption that $v|k,q\sim v|k$, $\Psi_GL(q,\dd v)$ models the conditional probability of $v|q$. Moreover, the moment projection operator effectively reduces this to a measure concentrated on a single point, i.e. $\E[V|q]$, which is consistent with the existing interpretation of attention.
    \end{remark}
    
    \paragraph{Recovering the Traditional Definition.}
    We have introduced a framework for attention-based models that uses measure theory and Markov kernels as the principal building blocks, and we have shown that it is equivalent (i.e. Proposition~\ref{prop:measure_attn}). It is reasonable to wonder if there is a way to recover the traditional linear-algebraic definition of attention from our framework, and the answer is yes.\nnewline
    
    Before proceeding, recall three basic facts about Markov kernels and measures on \emph{discrete} spaces:
    \begin{enumerate}\setlength{\itemsep}{0.05em}
         \item probability measures are stochastic vectors, 
        \item Markov kernels are stochastic matrices, and
        \item integration against these kernels is matrix multiplication.
    \end{enumerate}
    
    Suppose now that $E$ is the discrete set $E=\{1,\dots,N\}$ and let $Q=\{q_1,\dots,q_N\}$, $K=\{k_1,\dots, k_N\}$, and $V=\{v_1,\dots,v_N\}$ be subsets of $\R^d$. Then we set $G:Q\times V\to \R_{\geq 0}$ to be $G(i,j):=\exp\ip{Q[i]}{K[j]}$, $L:K\times 2^V$ is the (discrete) Markov kernel $L(i,\dd j)=\delta_{i}(\dd j)$ and $\Pi:\calP(V)\to \R^d$ is defined by $ \Pi(\mu):=\sum_{j=1}^N V[j]\mu_j=:V'[i]$. Then the analog of the attention kernel is the attention \emph{map}
    \[
        \mbf{A}_{m(K)}(i):=\Pi[\Psi_{G(i,\bdot)}(m(K))L].
    \]
    
    Having now defined discrete analogs of the components of attention defined in this section, we can make the following remarks. Firstly, $\delta_{q_i}(\dd j)$ is the Kronecker delta $\delta^j_i$ and $m(K)=\frac{1}{N}\mbf{1}$ where $\mbf{1}:=[1,\dots,1]$ (i.e. $N$ times). Secondly, the Boltzmann-Gibbs Markov kernel $\Psi_{G(i,\bdot)}(\dd j)$ is then equal to the usual softmax definition:
    \[
        \Psi_{G(i,\bdot)}(m(K))(\dd j)=\softmax(Q[i]K^T).
    \]The softmatch kernel is the stochastic matrix $\Psi_{G}(m(K))=\softmax(QK^T)$. Thirdly, composition of kernels is matrix multiplication, so $\Psi_{G}(m(K))L=\softmax(QK^T)Id=\softmax(QK^T)$ since $L$ is the $N\times N$ identity matrix $Id$. Finally, the attention map corresponds to the matrix multiplications
    \begin{align*}
        \delta_{Q[i]}\mbf{A}_{m(K)}&=\Pi[\Psi_{G(i,\bdot)}(m(K))L]\\
        &=\sum_{j=1}^n \Psi_{G(i,\bdot)}(m(K))(\dd j)V[j]= \sum_{j=1}^n \softmax(Q[i]K^T)[j]V[j]=\softmax(Q[i]K^T)V.
    \end{align*}
    This is the attention definition from \cite{bahdanau2014neural} and is used widely throughout the machine learning community.\nnewline
    
    It is perhaps reassuring that the constructions we have proposed are not so exotic as they might appear at first glance. However, we feel that using measures on the representation space rather than the index space offers significant gains. The reason for this is that distributions on $\R^d$ are much more expressive than distributions on $\{1,\dots,N\}$, and therefore admit more interesting analysis; see for example  our analysis of the regularity of attention in Section~\ref{sec:stability} relies on the $1$-Wasserstein distance, which is trivial in the discrete case.

\subsection{Extension to the Transformer}
We now sketch how to extend the measure-theoretic model of self-attention described in the previous section to the popular Transformer encoder architecture \citep{vaswani2017attention}. It is a straightforward application of the techniques above. We only describe here how our framework can model a single head Transformer\footnote{We only consider the encoder part of the transformer, since it uses self-attention. Our framework is fully compatible with the cross-attention from the transformer decoder \citep{vaswani2017attention}.}, and refer the interested reader to  Appendix~\ref{app:transformer} for the extension to a full multi-headed transformer. We seek to model
\begin{align}\label{eq:transformer}
    	\transf(X) = \ffn \circ \selfattn(X),
\end{align}
where $X=\{x_1,\dots,x_N\}\subset \R^d$ is the input data, $\selfattn(\bdot)$ is the scaled dot-product attention \citep{vaswani2017attention} and $\ffn(\bdot)$ represents a feedforward neural network. We set $\wtilde{a}(x,y)=x^Ty/\sqrt{d}$ and let
\[
	a(x,y) = \wtilde{a}\left((W^Q)^Tx, W^Ky\right),~~~~L(k_i,\dd v)=\delta_{W^Vv_i}(\dd v),
\]where $W^Q,W^K,W^V$ are matrices in $\R^{d'\times d}$. These correspond to the various matrix operations performed by the Transformer. We let $f:E\to E$ be the FFN in~\ref{eq:transformer} and define the FFN kernel as $\mbf{F}(x,\dd y)=\delta_{f(x)}(\dd y)$. Using the attention kernel $A$ from Definition~\ref{defn:attention_kernel}, we define $\mbf{T} := \mbf{A}\mbf{F}$, and show in the proposition below that $\mbf{T}$ implements the self-attention transformer (see Appendix~\ref{app:transformer} for the proof).

\begin{restatable}{proposition}{measuretransformer}\label{prop:measure_transformer}
	Let $X=\{x_1,\dots,x_N\}\subset\R^d$ be a collection of inputs. The nonlinear Markov transport equation $\delta_{x_i}\mapsto \delta_{x_i}\mbf{T}_{m(X)}$ implements the self-attention Transformer.
\end{restatable}

\section{Structural Analysis of Attention}\label{sec:analysis}

We have proposed a model of attention that encompasses transformers and is amenable to mathematical analysis. 
We will now exploit the model using tools from analysis on the space $\calP(E)$ to study attention in various contexts. 
We first seek to understand ways in which $\Psi_G$ and $\Pi$ implement a feature extractor from a collection of inputs $\mbf{x}=\{x_1,\dots,x_T\}$ (the deterministic lookup-table implementation of the lookup kernel is not relevant to this analysis, WLOG assume it is the identity).

\paragraph{Background.} We will work with the Wasserstein metric on $\calP(E)$. Let $\calP_1(E)$ be the set of probability measures with finite 1st moment.  
The {\bf 1-Wasserstein distance} between $\mu,\nu\in \calP_1(E)$ is 
\[
	\W_1(\mu,\nu):=\inf_{\pi\in \calC(\mu,\nu)} \iint_{E\times E}\|x - y\|_1 \pi(\dd x, \dd y)
\] where $\calC(\mu,\nu)$ is the set of distributions on $(E\times E,\calE\times\calE)$ with marginals $\mu,\nu$ on the first and second components, respectively. $\W_1$ is a metric on $\calP_1(E)$ which turns the pair $\calW_1:=(\calP_1(E),\W_1)$ into a complete, separable metric space \cite[ Ch 6]{villani2008optimal} called the \textbf{1-Wasserstein space}. We equip $\calW_1$ with the Borel $\sigma$-algebra induced by $\W_1$. \nnewline

We will also need the generalized entropy functional: let $\nu\in \calP(E)$ and define $H_\nu:\calP(E)\to \R\cup\{\infty\}$ by
\[
	H_\nu(\mu) := \rl{
		-\KL(\mu\|\nu)& \text{ if }\mu\ll \nu\\
		\infty & \text{ otherwise}.
	}
\]
where $\KL(\mu\|\nu) := \int \log\left(\diff{\mu}{\nu}\right) \dd \nu$ denotes the Kullback-Leibler divergence. This cost functional on $\calP(E)$ generalizes the entropy functional to non-uniform base measures.

\begin{definition}[Measure Convolution]
	Let $\mu,\nu\in \calP(\R^d)$ be measures. Then we define the \defn{measure convolution} of $\mu$ and $\nu$, denoted $\mu*\nu$, by
	\[
		(\mu*\nu)(f) := \int f(x+y)\mu(\dd x)\nu(\dd y)~~~\forall f\in \calB_b(\R^d)
	\]
	where $\calB_b(\R^d)$ denotes the space of bounded, measurable functions from $\R^d$ to $\R$.
	\end{definition}

\subsection{The Boltzmann-Gibbs Transformation $\Psi_G$}

Recall that a collection of inputs $X=\{\delta_{x_1},\dots,\delta_{x_N}\}$ has joint configuration $m(X)$, and that the attention kernel $\mbf{A}$ produces new representations $\delta_{x_i}\mbf{A}_{m(X)}$ and a new joint configuration  $m(X)\mbf{A}_{m(X)}$ by nonlinear Markov transport\footnote{By nonlinear Markov transport, we mean that the mass of a measure $\mu$ is transported by a Markov kernel $\mbf{A}$ which \emph{depends on the measure} $\mu$, i.e. $\mu\mapsto \mu\mbf{A}_\mu$}. These new measures have encoded in them features about the inputs; our goal is to understand this process.\nnewline

Recall from Section~\ref{sub:features} that we say a measure $\mu\in \calP(E)$ encodes a feature $f\in E'$ w.r.t. a function $F:E\to E'$ if $\mu(F)=f$. Self-attention has the ability to extract features that depend on the inputs --- this was the original motivation for the attention mechanism in \cite{bahdanau2014neural}. Therefore we replace the function $F$ from above with a bi-measurable map $k:E\times E\to E'$ so that for any fixed $x\in E$,  $\mu$ encodes a feature $f(x)\in E'$ w.r.t. $k(x,\bdot)$, i.e. $\mu(k(x,\bdot))=f(x)$.\nnewline

Suppose now that $m(X)$ and $x$ are fixed, and that $\Psi_{G(x,\bdot)}(m(X))$ encodes $f(x)$ w.r.t. $k(x,\bdot)$. Then we will show that the measure $\Psi_{G(x,\bdot)}(m(X))$ is actually optimal among \emph{all} measures $\mu_x\ll m(X)$ which encode $f(x)$ w.r.t. $k(x,\bdot)$ in the maximum entropy sense. This property characterizes the Softmatch kernel. \nnewline

In this section and the next, we make the assumption below, which will hold unless otherwise stated.
\begin{assumption}\label{assump:1}
	$E\subset\R^d$ is a compact, convex set\footnote{Convexity is assumed only to guarantee that the moment-projection onto $\calP_\delta(E)$ is well-defined.}.
\end{assumption}

\begin{restatable}{theorem}{maxent}\label{thm:maxent}
	Let $k:E\times E\to E'$ be a bi-measurable bounded function, $f:E\to E'$ a measurable function, and $\nu\in \calW_1$ be fixed. Define the set
	\[
		\calQ(\nu,x):=\{\mu\in \calW_1\st \mu\ll \nu,~~~\mu(k(x,\bdot)) = f(x)\}, ~\forall x\in E.
	\] Then for each $x\in E$, there exists a unique maximizer $\mu^*_x$ to the maximum entropy problem
	\begin{equation}\label{eq:maxent}
		\max_{\mu_x\in \calQ(\nu,x)} H_\nu(\mu_x)
	\end{equation}
	given by $\mu^*_x=\Psi_{G(x,\bdot)}(\nu)$ with $G(x,y)=\exp\ip{\lambda(x)}{k(x,y)}$ for some $\lambda(x)\in \R^{d'}$. Furthermore, $x\mapsto \mu^*_x$ is measurable, so $(x,\dd y)\mapsto \mu^*_x(\dd y)$ is the Markov kernel $(x,\dd y)\mapsto\Psi_{G(x,\bdot)}(\nu)(dy)$.
\end{restatable}

\begin{proof}[Proof Sketch]
	See Appendix~\ref{proof:maxent} for the full proof. Existence and uniqueness of $\mu^*_x$ are classical results. To ensure that $\mu^*_x$ coincides with the Markov kernel $\Psi_{G(x,\bdot)}(\nu)$ we use the additional result that $x\mapsto \mu^*_x$ is measurable which is a technical measurable selection theorem in the appendix. 
\end{proof}

The maxent property is not particularly surprising, given the well-known properties of the softmax and maximum entropy, but it is not \emph{a priori} clear that the dependence of such a measure on $x$ should posess any good properties. We show that the pointwise maximum entropy measure depends measurably on $x$ if $k,f$ are measurable, so the former is indeed the the softmatch kernel\footnote{This is actually a stronger result --- to be a Markov kernel, we need $x\mapsto \mu^*_x(A)$ to be measurable for each $A\in \calE$, but we have actually shown that $x\mapsto \mu^*_x$ as a mapping from $E\to \calP(E)$ is measurable.}.

\begin{remark}
	The minimization problem \eqref{eq:maxent} is an instance of the ``information projection'' operator \citep{koller2009probabilistic}. Below, we see how it relates to the dual notion of moment projection.
\end{remark}

\subsection{The Moment Projection $\Pi$}
Without the moment projection $\Pi$ in the definition of the attention kernel $\mbf{A}$, successive applications of $\mbf{A}$ would result in a measure of arbitrary complexity. On the other hand, we work with finite-dimensional representations, so using the entire space $\calP(E)$ is simply too large to be practical. Therefore, we view $\Pi_\calF:\calP(E)\to \calF$ as a crucial element of attention which ensures that the output measure remains on a finite-dimensional subspace of $\calP(E)$ while preserving moments.\nnewline 

We study two choices of $(\calF,F)$. The first is when $\calF$ is an exponential family on $E$ with sufficient statistics $F$ and base measure $\lambda$, $\calF:=\left\{\nu_\theta(\dd x ) = e^{\ip{\theta}{F(x)} - A(\theta)}\lambda(\dd x) \st \theta\in \R^{d'}\right\}$ and the second is $\calF = \calP_\delta(E)$ with $F(x)=x$. Our main result in this section is a connection between the ``synthetic'' definition of $\Pi$ with an ``analytic'' description as a minimization problem.

\begin{restatable}{proposition}{expmomproj}\label{prop:expmomproj}
	Let $E \subseteq\R^d$ and $\calF$ be a natural exponential family with sufficient statistics $F:E\to E'\subseteq\R^{l}$ and base measure $\lambda$ supported on $E$ as above. Suppose that, for $\mu\in \calP(E)$, $\mu\sim \nu,~\forall \nu\in \calF$. Then $\Pi:\calP(E)\to \calF$ satisfies the projection-like equation
	\[
		\Pi(\mu) = \arg\min_{\nu \in \calF}\KL(\mu\|\nu).
	\]
\end{restatable}

See Appendix~\ref{proof:expmomproj} for the proof, which is standard. Extending the result to the case of $\calF=\calP_\delta(E)$ and $F(x)=x$ requires different arguments because $\mu\not\ll\nu$ in general. Intuitively, we can think of $\calP_\delta(E)$ as a the limit of exponential family $\calF_\sigma:=\{\calN(\dd x,\sigma)\st x\in E\}$ as $\sigma\to 0$. Our approach leverages properties of $F$ and convolution with $\calN(\dd x,\sigma)$ to regularize $\calF$ and $\Pi$ while preserving the encoded moments, then take a limit as $\sigma\to 0$.

\begin{restatable}{proposition}{deltamomproj}\label{prop:deltamomproj}
	Let $\mu\in \calW_1$, and $F(x) = x$. Setting $\calF_n :=   
	\{\rho_n * \delta_x\st \delta_x\in \calP_\delta(E)\}$ $(= \{\calN(x,\sigma_n I), x\in E\})$, where 
	$\rho_n:=\calN(0,\sigma_n I)$ 
	with $\sigma_n\to 0$, then in the 1-Wasserstein topology:
	\[
		\Pi_\calF(\mu) = \lim_{n\to \infty}\arg\min_{\nu\in \calF_n}\KL(\rho_n*\mu\|\nu).
	\]
\end{restatable}

\begin{proof}[Sketch]
    The idea of the proof is to regularize both $\mu$ and $\calF$ by convolution with a sequence of Gaussians $\rho_n$, thereby ensuring that $\rho_n*\mu$ and $\nu\in \calF_n$ have the same supports without affecting the encoded moments.  Theorem~\ref{prop:expmomproj} then applies, and the we show that the sequence of regularized minimizers tends to $\Pi_\calF(\mu)$.
\end{proof}
\begin{remark}
	On a practical note, the choice of $\calF$ can greatly affect how efficient $\Pi$ is to compute. When $\calF = P_\delta(E)$ and $F(x)=x$, then this reduces to linear algebra operations since elements of $\calP_\delta(E)$ (and mixtures thereof) have easy closed-form moments. On the other hand, if $\calF$ is an arbitrary exponential family, e.g. Gaussians with $F(x)=[x\|x^2]$, then a closed form solution to $\mu(F)$ or $\Psi_{G}(\mu)$ may not exist, so one is reduced to computing the projection and/or the normalizing constant via approximation.
\end{remark}

\subsection{Interaction Between the Components}
	We have seen how the transformation $\Psi_G$ extracts a set of features determined by $G$ and encodes them as moments into $\mu=\Psi_G(\nu)$ in a maximum-entropy transformation. The cost to pay for this encoding is an expansion of the degrees of freedom of the measure, i.e. $\mu$ no longer lives in the finite-dimensional moment subspace $\calF_F$. On the other hand, the projection $\Pi$ ``squeezes'' the measure $\mu$ through a bottleneck, reducing the degrees of freedom to those determined by $\calF_F$. This is a accomplished via the moment projection, and can be seen as discarding extra information in the system while preserving a set of moments determined by $F$. 

\section{Regularity of Attention}\label{sec:stability}

In this section, we consider self-attention as a non-linear map from $\calP(E)$ to $\calP(E)$ through $\mbf{A}:\mu \to \mu \mbf{A}_\mu$, and derive a contraction estimate for this map on the metric space $(\calP_1(E),\W_1)$  via an inequality of the form:
\[
	\sup_{\mu\neq \nu}\W_1(\mu\mbf{A}_\mu,\nu\mbf{A}_\nu)\leq \tau(\mbf{A}) \W_1(\mu,\nu)
\]for some constant $\tau(\mbf{A})$ to be determined. These estimates naturally lead to sufficient conditions on $G,L$ relative to $E$ which guarantee the stability of self-attention networks. Recall that Assumption~\ref{assump:1} is in effect.\nnewline

Our approach is to estimate contraction coefficients for $\Psi_G,L,\Pi$ separately and then combine them. We will need the following notions: for a real-valued function $f$, $\|f\|_{Lip}$ is the Lipschitz semi-norm $\|f\|_{Lip}:= \sup_{x \neq y} |f(x) - f(y)| / d(x,y)$. For a function $G$ of two variables, $G: E \times E \to \R$, we let $\|G\|_{Lip,\infty} := \sup_{x\in E}\|G(\bdot,x)\|_{Lip}$ and $\|G\|_{\infty,Lip} := \sup_{x\in E}\|G(x,\bdot)\|_{Lip}$.

We now prove contraction estimates on the various components of attention. Their proofs can be found in Appendix~\ref{app:stab_proofs}, they mainly rely on the Kantorovich dual formulation of $\W_1$.
\begin{restatable}{proposition}{psicontract}\label{prop:psi_contract}
	Suppose $\mu,\nu\in \calP_1(E)$ and $G:E\times E\to \R$,  $G(x,y)\geq \epsilon(G) > 0$ is an interaction potential s.t. $\|G\|_{\infty,Lip} < \infty$ and $\|G\|_{Lip,\infty} <\infty$. Then, $\forall x,y \in E$:
    \begin{align*}
		&\W_1(\Psi_{G(x,\bdot)}(\mu), \Psi_{G(y,\bdot)}(\mu)) \leq 2 \frac{\|G\|_{Lip,\infty}\diam(E)}{\epsilon(G)} \cdot d(x,y), \\
		&\W_1(\Psi_{G(x,\bdot)}(\mu),\Psi_{G(x,\bdot)}(\nu)) \leq \frac{2\|G\|_{\infty,Lip}\diam(E)}{\epsilon(G)}\cdot \W_1(\mu,\nu).
    \end{align*}
\end{restatable}

\begin{restatable}{proposition}{picontract}\label{prop:pi_contract}
	Suppose that $F(x)=x$, $\calF=\calP_\delta(E)$, and $\Pi:=\Pi_\calF$ is the moment projection onto $\calF$. Then, for $\mu,\nu\in \calP_1(E)$, $\W_1(\Pi(\mu), \Pi(\nu))\leq d \cdot \W_1(\mu,\nu)$.
\end{restatable}

\begin{restatable}{proposition}{lookupcontract}\label{prop:lookup_contract}
	Suppose $L:E\times \calE\to [0,1]$ is a lookup kernel implementing a deterministic lookup function $\ell:E\to E$, (i.e. $L(x,\dd y)=\delta_{\ell(x)}(\dd y)$) and suppose that $\ell$ is $K_\ell$-Lipschitz in the 1-norm, then $\W_1(\mu L, \gamma L)\leq K_\ell \W_1(\mu,\gamma)$.
    
\end{restatable}

Combining the propositions allows to bound the Wasserstein contraction coefficient of the attention kernel $\mbf{A}$, defined as:
	\begin{definition}[Wasserstein Contraction Coefficient]
		Let $\Phi:\calP_1(E)\to \calP_1(E)$ be a (possibly nonlinear) mapping. We define the \defn{Wasserstein contraction coefficient} by
		\[
			\tau(\Phi):=\sup_{\mu\neq \nu}\frac{\W_1(\Phi(\mu), \Phi(\nu))}{\W_1(\mu,\nu)}.
		\]
	\end{definition}

    \begin{restatable}{theorem}{contraction}\label{thm:contraction}		
        Let $E\subset \R^d$ be compact, $\calF=\calP_\delta(E)$, and $\Pi:=\Pi_\calF$. Let $\mbf{A}$ be the attention kernel from definition~\ref{defn:attention_kernel}, with $L(x,\dd y)=\delta_{\ell(x)}(\dd y)$. Let $G$ be an interaction potential such that $G(x,y)\geq \epsilon(G) > 0$, $\|G\|_{Lip,\infty} <\infty$ and $\|G\|_{\infty,Lip} <\infty$. Then, with $\tau(\mbf{A})$ the 1-Wasserstein contraction coefficient of $\mbf{A}$ considered as a mapping on $\calP(E)$ through $\mbf{A}:\mu \to \mu \mbf{A}_\mu$, we have:
		\[
			\tau(\mbf{A})\leq \tau(\Pi)\tau(\Psi_G)\tau(L)
		\]where $\tau(\Pi) = d$, $\tau(\Psi_G) = \frac{2(\|G\|_{Lip,\infty} + \|G\|_{\infty,Lip})\diam(E)}{\epsilon(G)}$ and $\tau(L) = K_\ell$.

	\end{restatable}
	\begin{proof}
		See Appendix~\ref{app:stab_proofs}. 
	\end{proof}

\begin{corollary}
	Let $K=\{k_1,\dots,k_N\}\subset E \subset\R^d$ and $V=\{v_1,\dots,v_n\}\subset E\subset \R^d$ and the attention function $\attn(\bdot,K,V)$ be as in the original defintion of attention from \cite{bahdanau2014neural}, Definition~\ref{defn:attention}. Assume that the components of $\attn(\bdot,K,V)$ satisfy Theorem~\ref{thm:contraction}. Then the mapping
	\[
		q\mapsto \attn(q,K,V)
	\]is Lipschitz continuous as a mapping from $\R^d\to \R^d$ with the Euclidean distance, and moreover 
	\[
		\|\attn(q_1,K,V) - \attn(q_2,K,V)\|_2 \leq d^{3/2}\cdot \|\ell\|_{Lip}\cdot \frac{2\|G\|_{Lip,\infty}\diam(E)}{\epsilon(G)} \cdot \|q_1 - q_2\|_2
	\]
\end{corollary}

\begin{proof}
	Using proof from Theorem~\ref{thm:contraction} in Appendix~\ref{app:stab_proofs} we have
	\begin{align*}
		\|\attn(q_1,K,V) - \attn(q_2,K,V)\|_2&\leq \|\attn(q_1,K,V) - \attn(q_2,K,V)\|_1\\
		&=\W_1(\delta_{q_1}\mbf{A}_{m(K)},\delta_{q_2}\mbf{A}_{m(K)})\\
		&\leq d\cdot \|\ell\|_{Lip}\cdot \frac{2\|G\|_{Lip,\infty}\diam(E)}{\epsilon(G)} \cdot \W_1(\delta_{q_1},\delta_{q_2})\\
		&=  d\cdot \|\ell\|_{Lip}\cdot \frac{2\|G\|_{Lip,\infty}\diam(E)}{\epsilon(G)} \cdot \|q_1 - q_2\|_1\\
		&=  d^{3/2}\cdot \|\ell\|_{Lip}\cdot \frac{2\|G\|_{Lip,\infty}\diam(E)}{\epsilon(G)} \cdot \|q_1 - q_2\|_2
	\end{align*}
	using $\|x\|_2\leq \|x\|_1\leq \sqrt{d}\|x\|_2$ for norms on $\R^d$.
\end{proof}
The reader should note that the dependence of the above estimate on the geometry of $V$ w.r.t. $K$ is contained in the Lipschitz constant $\|\ell\|_{Lip}$. 

\section{Applications}\label{sec:applications}
In this section, we provide insights into existing applications of self-attention that one can gain using the analysis developed above.

\subsection{Attention is Continuous in the Input Data}
Let  $X=\{x_t\}_{t\in \calT}$  be a set of structured data. In many cases, such as graphs, there is a neighbourhood structure $\calN:\calT\to 2^\calT$ which effectively determines the statistical dependence of elements of different elements by $P(x_t|X_{-t}) = P(x_t|X_{\calN(t)})$.  Modelling $X$ amounts to picking a family of distributions $Q_\theta$ and a neighbourhood structure $\wat{\calN}$ to approximate $P(x_t | X_{-t})\approx Q_\theta(x_t|X_{\wat{\calN}(t)})$. Often, e.g. when $X$ represents a graph or an image, the choice of $\wat{\calN}$ is not obvious so a modelling decision needs to be made, sometimes to include only elements inside a ball around $t$ (if $\calT$ has a metric structure), or to randomly sample elements from $\calT$.\nnewline

Our framework provides techniques for understanding how this modelling decision impacts self-attention based models. In particular, the following results characterizes how modelling errors of the form $\wat{\calN}(t)\neq \calN(t)$ propagate through a self-attention network --- the output representations cannot be arbitrarily far from the ``true representations''.

\begin{proposition} 
    Suppose $X=\{x_t\}_{t\in \calT}$ is a set of structure data and let $t\in \calT$. Suppose $\calN(t),\wat{\calN}(t)\subset\calT$ are two neighbourhoods that both contain $t$. Denote by $X_{\calN(t)}:=\{x_t\}_{t\in \calN(t)}$ and similarly for $X_{\wat{N}(t)}$. Lastly, assume that the assumptions of Theorem~\ref{thm:contraction} hold. Then
    \begin{align*}
        &\|\attn(x_t,X_{\calN(t)},X_{\calN(t)}) - \attn(x_t,X_{\wat{\calN}(t)},X_{\wat{\calN}(t)})\|\\
        &\leq  d\cdot \tau_1(L)\frac{2\|G(x_t,\bdot)\|_{Lip}\diam(E)}{\veps(G)}\cdot \W_1(m(X_{\calN(t)}),m(X_{\wat{\calN}(t)})) 
    \end{align*}
\end{proposition}

\begin{proof}
    We can adapt an argument from the proof of Theorem~\ref{thm:contraction}. Firstly, for simplicity write $\mu_t:=m(X_{\calN(t)}),$ $\wat{\mu}_t:=m(X_{\wat{\calN}(t)})$ and note that
    \[
        \|\attn(x_t,X_{\calN(t)},X_{\calN(t)} - \attn(x_t,X_{\wat{\calN}(t)},X_{\wat{\calN}(t)})\| = \W_1(\delta_{x_t}\mbf{A}_{\mu_t}, \delta_{x_t}\mbf{A}_{\wat{\mu}_t}).
    \]Then 
    \begin{align*}
        \W_1(\delta_{x_t}\mbf{A}_{\mu_t}, \delta_{x_t}\mbf{A}_{\wat{\mu}_t})
        &= \W_1(\delta_{x_t}\Pi[\Psi_{G(\bdot,\bdot)}(\mu_t)L], \delta_{x_t}\Pi[\Psi_{G(\bdot,\bdot)}(\wat{\mu}_t)L]) \\
        &= \W_1(\Pi[\Psi_{G(x_t,\bdot)}(\mu_t)L], \Pi[\Psi_{G(x_t,\bdot)}(\wat{\mu}_t)L]) \\
        &\leq \tau_1(\Pi) \tau_1(L) \W_1(\Psi_{G(x_t,\bdot)}(\mu_t), \Psi_{G(x_t,\bdot)}(\wat{\mu}_t))\\
        &\leq \tau_1(\Pi) \tau_1(L) \frac{2\|G(x_t,\bdot)\|_{Lip}\diam(E)}{\veps(G)}\W_1(\mu_t,\wat{\mu}_t)\\
        &= d\cdot \tau_1(L)\frac{2\|G(x_t,\bdot)\|_{Lip}\diam(E)}{\veps(G)}\W_1(\mu_t,\wat{\mu}_t).\qedhere
    \end{align*}
\end{proof}

Let us make a few remarks on this result. Firstly, it is automatically \emph{permutation invariant}, w.r.t. the input data $X_{\calN(t)}$. This property is a consequence of the choice to use measure theory and represent the input data using a measure. 
In the case that $|\calN(t)|=|\wat{\calN}(t)|=m$, we can obtain an explicit formula for $\W_1(\mu,\wat{\mu})$ (see e.g. \cite{bobkov2014one}~Lemma 4.2)
\[
    \W_1(\mu,\wat{\mu}) = \inf_{\sigma\in \Sigma(m)} \frac{1}{m}\sum_{i=1}^m \|x_s - y_{\sigma(s)}\|_1
\] 
where $x_s \in X_{\calN(t)}$, $y_s\in X_{\wat{\calN}(t)}$ and $\Sigma(m)$ is the set of permutations on $m$ elements.

\subsection{Infinitely Deep Transformers}
Using our framework, we can also investigate some properties of infinitely-deep transformers. Indeed, the work by \cite{dehghani2018universal} describes the ``Universal Transformer'' which processes a collection of inputs using self-attention in a recurrently-connected model. This model can be thought of as an infinitely-deep self-attention transformer with weight sharing between the layers, and our quantitative contraction estimates from Section~\ref{sec:stability} provide an opportunity to study the Universal Transformer's behaviour as the number of processing steps (i.e. layers) tends to $\infty$. In particular, one can consider the model
\[
    X^{\ell + 1} = \transf(X^\ell), ~~~X=\{x_1,\dots,x_N\}\subset\R^d
\]where $\transf(\bdot)$ is defined in Appendix~\ref{app:transformer}. One natural question is: does there exist a limit to this recurrence? This amounts to having a fixed-point for the transformer mapping. 
\begin{proposition}\label{prop:banach}
    Let $\tau(\mbf{T})$ be the contraction coefficient for the self-attention Transformer. Then if $\tau(\mbf{T})<1$, there exists a unique fixed point to the Universal Transformer recurrence.
\end{proposition}
\begin{proof}
    This is a straightforward application of the Banach fixed point theorem. One can calculate the exact requirements on the constituents of the nonlinear Markov kernel $\mbf{T}$ in the same way as Theorem~\ref{thm:contraction}.
\end{proof}
\begin{remark}
    We should note that the result from Proposition~\ref{prop:banach} is quite strong --- the fixed point we describe is the same for all input data $X$. It is not clear what properties this fixed point might have, and whether they are interesting, but we provide this result as a possible line of future work.
\end{remark}

\subsection{Lipschitz Contstants for $E=\R^d$}\label{subsec:unbounded}
The results of Section~\ref{sec:stability} depend on the boundedness of the representation space $E$. While this enables us to provide rather general estimates on the Lipschitz coefficient for attention that are verified by most reasonable choices of $G,L$, it is reasonable to question how realistic this assumption is, and whether it can be relaxed. \nnewline

In concurrent work by \cite{kim2020lipschitz}, the authors investigate Lipschitz constants for self-attention on $X=\{x_1,\dots,x_N\}$ as a mapping from $\R^{d\times N}\to \R^{d\times N}$ without assuming $E$ is bounded. The authors show that, for the case of $G(x,y)=\exp\ip{x}{y}$ on the whole of $\R^d$, attention is not Lipschitz by showing the norm of Jacobian is unbounded (\cite{kim2020lipschitz}~Theorem 3.1). The authors characterize this failure (adapted to our notation) as being due to the unbounded variance of $m(X)\Psi_{G}(m(X))$. Moreover, they show that using (a variation of) $G(x,y)=\exp(-\|x-y\|_2^2)$ as an interaction potential leads to a uniform Lipschitz bound (\cite{kim2020lipschitz}~Theorem 3.2). The authors provide empirical evidence that this potential function does not severely degrade performance. \nnewline

As further evidence of the strength of our framework, we provide below an analysis of the same basic Gaussian interaction potential $G(x,y)=\exp(-\|x-y\|_2^2)$ as in \cite{kim2020lipschitz}\footnote{We chose the un-parameterized potential for simplicity, we see no reason our framework would not extend to the parameterized case as well.} for unbounded $E=\R^d$. We are able to use the same basic tools ($1$-Wasserstein contractions, estimation techniques from analysis, etc) and a similar approach to Section~\ref{sec:stability} but exchanging boundedness assumptions on $E$ for exponential decay of $G(x,y)$ and $\|\nabla G(x,y)\|$ as $\|x-y\|_2\to \infty$. The proofs can be found in Appendix~\ref{app:unbounded}.

\begin{restatable}{proposition}{psiContractUnbounded}\label{prop:psi_contract_unbounded}
	Let $E$=$\R^d$ and suppose $X=\{x_1,\dots,x_N\}$ and $Y=\{y_1,\dots,y_N\}$ Let $\mu=m(X),~\nu=m(Y)$. Then for $x\in \supp{\mu}$ and $y\in \supp{\nu}$, we have

\[
        \W_1(\Psi_{G(x,\bdot)}(\mu),\Psi_{G(y,\bdot)}(\nu))\leq \left[\sqrt{d} \sqrt{\ln{N} +\frac{1}{2e}}\|G\|_{Lip} + \|G\|_{\infty} + \sqrt{d} + 2\right](d(x,y) + \W_1(\mu,\nu)).
\] 
\end{restatable}

\begin{restatable}{theorem}{tauAunbounded}\label{thm:tau_a_unbounded}
    Let $E$=$\R^d$ and suppose $X=\{x_1,\dots,x_N\}$ and $Y=\{y_1,\dots,y_M\}$. Let $G(x,y)=\exp(-\|x-y\|^2_2),~L(x,\dd y)=\delta_{l(x)}(\dd y)$, and $\Pi$ be the usual projection onto $\calF_\delta$. Then for $\mu=m(X)$ and $\nu=m(Y)$,
    \[
     \W_1(\mu\mbf{A}_\mu, \nu\mbf{A}_\nu)\leq 2\tau(\Pi)\tau(L) \left[\sqrt{d} \sqrt{\ln (\min(N,M)) +\frac{1}{2e}}\|G\|_{Lip} + \|G\|_{\infty} + \sqrt{d} + 2\right]\W_1(\mu,\nu).
    \]
\end{restatable}

Theorem~\ref{thm:tau_a_unbounded} provides an alternate path to the Lipschitz constant of self-attention compared to methods based on computing Jacobians~\citep{kim2020lipschitz}. In particular, it applies to sequences of tokens of various lengths and allows to study the effect of perturbing a sequence by e.g. removing a given word, or negating a sentence, which is out of immediate reach for Jacobian-based techniques. Finally, we can recover a bound for sequences of equal lengths:

\begin{restatable}{corollary}{unequal}\label{coro:unequal}
	Applying Theorem~\ref{thm:tau_a_unbounded} to the case of $N = M$ gives:
	\[
		\W_1(\mu\mbf{A}_\mu,\nu\mbf{A}_\nu)\leq 2\tau(\Pi)\tau(L)\left[\sqrt{d} \sqrt{\ln{N}+\frac{1}{2e}}\|G\|_{Lip} + \|G\|_{\infty} + \sqrt{d} + 2\right]\W_1(\mu,\nu).
	\]
\end{restatable}

Our bound appears to have a better dependence on $N:=|X|$ (although perhaps at a worse dependence on $d$) than in \cite{kim2020lipschitz}; ours is $\calO(\sqrt{\ln N})$ and theirs is at best $\calO(\ln(N))$. 

\paragraph{Negated Sentences.} Our regularity results provide a theoretical framework for recent observations on the behavior of deep language models with respect to negation. Table 4 of \cite{kassner2019negated} shows that negated sentences are often given identical predictions to the original ones: for instance, both ``A beagle is a type of [MASK]'' and ``A beagle is not a type of [MASK]'' get a prediction of ``dog''. \nnewline

One way to analyze this failure mode (and others like it) is to study how the geometry of the contextual embeddings changes under negation; if the embeddings are close with and without negation (i.e. the model is``too smooth'' w.r.t. perturbations in token space) the scoring network (often a linear classifier) will not be able to distinguish between the resulting embeddings and so the model will fail. Our modelling is based on the Wasserstein distance, which naturally handles sentences of different lengths, so it could be used to derive testable predictions of the distance between a self-attention networks' contextual embeddings \emph{as a function of the context} (e.g. for sentences with and without a ``not''). We leave that research direction to future work.

\section{Related Work}\label{sec:related}
Given the somewhat unorthodox approach to studying attention taken in this paper, we will briefly overview two very different areas that are related to this work: attention models (with an emphasis on variety of domains where attention is used) and self-interacting measure-valued equations.\nnewline

We first highlight some important successes of attention-based models across domains of machine learning. As noted by \cite{smola2019attention}, the original notion of attention appears in statistics in the form of the Watson-Nadaraya estimator \citep{watson1964smooth, nadaraya1964estimating} which implements a data-dependent regression model. The term ``attention'' and the modern ``query-key-value'' formulation comes from \cite{bahdanau2014neural} who use attention for sequence alignment in a recurrent neural translation model. A similar setup was used in \cite{graves2014neural} for differentiable, content-based addressing of a memory array. In \cite{sukhbaatar2015end} and \cite{seo2016bidirectional}, attention is used for question answering, machine reading comprehension, and language modelling. The extremely successful ``Transformer'' architecture was introduced in \cite{vaswani2017attention} and demonstrated that one could build powerful neural networks using attention as the main component. This led to important developments in language modelling~\citep{devlin2018bert,radford2018improving}, graph modelling \citep{velivckovic2017graph}, image modelling \citep{parmar2018image}, and set modelling \citep{lee2018set}. Recently, \cite{baker2019emergent} used attention in the policy architecture of a multi-agent reinforcement learning problem. \nnewline

Concurrent to our work, there has been a very recent flurry of activity in the study of the properties of attention-based networks from an empirical and theoretical perspective. As discussed in Section~\ref{sec:applications}, \cite{kim2020lipschitz} studies the Lipschitz constant of self-attention as a map from $\R^{d\times N}\to \R^{d\times N}$. In \cite{katharopoulos2020transformers}, the authors show that transformers with causal masks in the attention inputs can be written as a recurrent computation, linking this application with RNNs. In \cite{bhattamishra2020computational}, the authors study the computational power of Transformers as measured by Turing completeness. In \cite{hron2020infinite}, the authors study the behaviour of attention-based networks as the number of heads tends to infinity and show that the self-attention architecture behaves as a Gaussian process in the limit. Finally, in \cite{levine2020limits}, the authors study a reduced form of deep self-attention networks to understand the interplay between depth and width.\nnewline

The theory of self-interacting Markov chains and the associated measure-valued processes was developed extensively in \cite{moral2004feynman}, who studied so-called ``Feynman-Kac models'' using semigroup techniques. These systems model a wide array of self-interacting phenomena such as genetic algorithms, nonlinear filtering equations, self-avoiding random walks, self-interacting particle systems, etc.~\citep{moral2004feynman}. Such interacting particle systems have also appeared in the Markov Chain Monte Carlo literature \citep{andrieu2007non} and the mean-field game literature \citep{huang2006large,lasry2006jeux}. Finally, as we discussed in Remark~\ref{rem:interacting}, \cite{lu2019understanding} also study the Transformer as a system of interacting particles, albeit using tools from numerical ODEs rather than measure-valued flows, and derive practical insights.

\section{Conclusion}

In summary, we have presented a model of attention using measure theory which is mathematically equivalent to the original formulation of attention from \cite{bahdanau2014neural} and fully compatible with the self-attention Transformer from \cite{vaswani2017attention}. Moreover, this framework is flexible enough to model all of the input modalities where attention is used, from graphs to sentences.\nnewline

In exchange for the increased complexity of working with measures, we are able to connect self-attention to a maximum entropy problem and prove that, in a maximum-entropy sense, self-attention is an optimal feature extraction model. We also derive contraction estimates for self-attention in the 1-Wasserstein distance. We then apply these insights to study three practical problems from attention-based neural networks, namely continuity w.r.t. input data, infinitely-deep transformers, and Lipschitz estimates when $E=\R^d$ for a particular choice of interaction potential.  \nnewline

\clearpage

\bibliography{math_attentionV2}

\appendix

\section{The Transformer}\label{app:transformer}


We now show how to extend the measure-theoretic model of self-attention described in the main text to the full Transformer encoder architecture \citep{vaswani2017attention}\footnote{Technically, the Transformer also contains layer normalization and residual connections, which we do not treat here.}. This is a straightforward application of the techniques from the main text. For our purpose, we work with the model
\begin{align}\label{eq:transformer2}
    	\transf(X) = \ffn \circ \multiselfattn(X)
\end{align}
where $X=\{x_1,\dots,x_N\}\subset \R^d$ and $\ffn$ represents a feedforward neural network. To incorporate this into our formalism above, first set $\wtilde{a}(x,y)=x^Ty/\sqrt{d}$. We can model a single head of the Transformer using the attention kernel from Definition~\ref{defn:attention_kernel} with:
\[
	a(x,y) = \wtilde{a}\left((W^Q)^Tx, W^Ky\right),~~~~L(k_i,\dd v)=\delta_{W^Vv_i}(\dd v).
\]where $W^Q,W^K,W^V$ are matrices in $\R^{d'\times d}$. To model multi-headed attention, we note that multi-headedness amounts to processing independent copies of the data $X$ and combining them with concatenation and matrix multiplication. The ``concat-and-matmult'' operation can be written as 
\[
	\begin{bmatrix}
	 	x_i^1 & \cdots & x_i^H
	 \end{bmatrix}\begin{bmatrix}
	 	W^O_1 \\ \vdots \\ W^O_H
	 \end{bmatrix} = x_i^1W^O_1 + \cdots + x_i^HW^O_H
\] 
where each $W^O_h\in \R^{d'\times d}$. Hence, letting $\mbf{O}^h(x,\dd y):=\delta_{xW^O_h\cdot H}(\dd y)$, where we have multiplied by the scalar $H$, and introducing the mixture kernel
\[
	\wat{\mbf{M}} := \frac{1}{H}\sum_{h=1}^H \mbf{A}^h\mbf{O}^h
\]where each $h$ parameterizes its own collection of projection matrices, we can define the multi-headed attention attention kernel as
\[
	\mbf{M}:=\Pi\circ\wat{\mbf{M}},~~~\mbf{M}_\mu(x,\dd y) = \Pi(\wat{\mbf{M}}_\mu(x,\bdot))(\dd y).
\]Finally, letting $f:E\to E$ be the FFN in~\ref{eq:transformer2} and defining the FFN kernel as $\mbf{F}(x,\dd y)=\delta_{f(x)}(\dd y)$, we see that $\mbf{T} := \mbf{M}\mbf{F}$ implements the self-attention transformer as nonlinear measure transport.

\measuretransformer*
\begin{proof}
	Given the discussion about standard attention, the only new element to be checked is the multi-headed attention kernel. Consider a fixed $X$, then
	\[
		m(X)\mbf{M}_{m(X)}(\dd y) = \frac{1}{N}\sum_{i=1}^N \int \delta_{x_i}\mbf{M}_{m(X)}(x,\dd y) =  \frac{1}{N}\sum_{i=1}^N \mbf{M}_{m(X)}(x_i,\dd y)
	\]Hence considering a single $x_i$, we see that
	\[
		\mbf{M}_{m(X)}(x_i, \dd y) = \Pi\left(\wat{\mbf{M}}_{m(X)}(x_i,\bdot)\right)(\dd y).
	\]The inner kernel is
	\[
		\wat{\mbf{M}}_{m(X)}(x_i,\dd y) = \frac{1}{H}\sum_{h=1}^H \int\mbf{A}_{m(X)}^h(x_i,\dd z)\mbf{O}^h(z,\dd y) = \frac{1}{H}\sum_{h=1}^H \int\mbf{A}_{m(X)}^h(x_i,\dd z)\delta_{zW^O_h\cdot H}(\dd y).
	\]The measure $\mbf{A}_{m(X)}^h(x_i,\dd z)$ is a delta-measure concentrated on the point
	\[
		\sum_{j=1}^N \frac{\exp[\wtilde{a}((W_h^Q)^Tx_i,W_h^Kx_j]}{\sum_{p=1}^N\exp[\wtilde{a}((W_h^Q)^Tx_i,W_h^Kx_k)]} W_h^Vx_j = \multiselfattn(x_i,X,X)_h =: y^h_i
	\]hence
	\[
		\frac{1}{H}\sum_{h=1}^H \int\mbf{A}_{m(X)}^h(x_i,\dd z)\delta_{zW^O_h\cdot H}(\dd y) = \frac{1}{H}\sum_{h=1}^H \int\delta_{y^h_i}(\dd z)\delta_{zW^O_h\cdot H}(\dd y) = \frac{1}{H}\sum_{h=1}^H \delta_{y^h_iW^O_h\cdot H}(\dd y).
	\]Finally, applying the mapping $\Pi$ we get a measure that is concentrated on the point
	\begin{align*}
	    \int_E \frac{1}{H}\sum_{h=1}^H \delta_{y^h_iW^O_h\cdot H}(\dd y)y &= \frac{1}{H}\sum_{h=1}^H	y^h_iW^O_h\cdot H = \begin{bmatrix}
	 	y_i^1 & \cdots & y_i^H
	 \end{bmatrix}\begin{bmatrix}
	 	W^O_1 \\ \vdots \\ W^O_H
	 \end{bmatrix} \\
	 &= \multiselfattn(x_i,X,X),
	\end{align*}
	which concludes the proof.
\end{proof}

\section{Proofs From Section~\ref{sec:analysis}}\label{app:analysis_proofs}

\begin{lemma}\label{lem:psi_cont}
	Let $\nu\in \calP(E)$ be fixed, where $E\subset\R^d$ is compact. The mapping $\Psi_\bdot(\nu):L^1(\nu)\to (\calP_1(E),\W_1)$ is continuous.
\end{lemma}
\begin{proof}
	Suppose $\{F_n\}\subset L^1(\nu)$ converges to $G\in L^1(\nu)$, and let $f\in Lip(1)$. Then
	\begin{align*}
			|\Psi_G(\nu)(f) - \Psi_{F_n}(\nu)(f)| &= \left|\frac{\int f(x) G(x)\nu(\dd x)}{\nu(G)} - \frac{\int f(x)F_n(x) \nu(\dd x)}{\nu(F_n)}\right|\\
			&= \left|\int \left[\frac{G(x)}{\nu(G)} - \frac{F_n(x)}{\nu(F_n)}\right]f(x)\nu(\dd x)\right|\\
			&\leq \left\|\frac{G}{\nu(G)} - \frac{F_n}{\nu(F_n)}\right\|_{L^1(\nu)}\|f\|_\infty\\
			&\leq \left\|\frac{G}{\nu(G)} - \frac{F_n}{\nu(F_n)}\right\|_{L^1(\nu)} \diam(E).
		\end{align*}
		In the last line, we used the fact that $\Psi_{G}(\nu)(\bar{f}) = \Psi_{F_n}(\nu)(\bar{f})$ for any constant function $\bar{f}$. This allows us to subtract from $f$ any constant without changing the value of $|\Psi_{G}(\nu)(f) - \Psi_{F_n}(\nu)(f)|$. Hence, we can assume without loss of generality that $\|f\|_{\infty} \leq \diam(E)$ (picking an arbitrary $x \in E$, we have $\forall y \in E$, $|f(y) - f(x)| \leq |y - x| \|f\|_{Lip} \leq \diam(E)$).
		Taking the supremum over $f$ gives us $\W_1(\Psi_G(\nu), \Psi_{F_n}(\nu)) \leq \left\|\frac{G}{\nu(G)} - \frac{F_n}{\nu(F_n)}\right\|_{L^1(\nu)}$. Letting $n \to \infty$ concludes the proof.
\end{proof}

\begin{lemma}[\cite{himmelberg1976optimal}, Theorem 2]\label{lem:himmel1}
	Let $S$ and $A$ be Borel spaces, and $R$ be a Borel measurable compact-valued multi-function from $S\to A$ (i.e. $\forall s\in S, R(s) \subset A$ and is compact). With $GrR = \{(s,a) \in S\times A | a \in R(s)\}$, we let $u:Gr R\to \R$ be a Borel measurable function such that $u(s,\bdot)$ is an u.s.c. function on $R(s)$ for each $s\in S$. Then, there exists a Borel measurable selector $f:S\to A$ for $R$ such that
	\[
		u(s, f(s)) = \max_{a\in R(s)}u(s,a)~\forall s\in S.
	\]Moreover, the function defined by $v(s)=\max_{a\in R(s)}u(s,a)$ is Borel measurable.
\end{lemma}


\begin{lemma}[\cite{van2014renyi}, Theorem 19]\label{lem:kl_usc}
	Suppose that $\calX$ is a Polish space (which is true for $E$ in particular). Then for $P,Q\in \calP(\calX)$, $\KL(P\|Q)$ is a lower semi-continuous function of the pair $(P,Q)$ in the weak topology.
\end{lemma}

\begin{lemma}[\cite{aliprantis2013infinite}, Th. 18.17 p.603 (Filippov's Implicit Function Th)]\label{lem:filippov}
	Let $(S,\Sigma)$ be a measurable space and let $X,Y$ be separable measurable spaces. Suppose that $f:S\times X\to Y$ is a Carath\'eodory function and that $\vphi:S\to X$ is a weakly measurable correspondence with nonempty compact values. Assume that $\pi:S\to Y$ is a measurable selector from the range of $f$ on $\vphi$ in the sense that $\pi$ is measurable and for each $s\in S$ there exists $x\in \vphi(s)$ with $\pi(s)=f(s,x)$. Then the correspondence $\nu:S\to X$ defined by
	\[
		\nu(s) = \{x\in\vphi(s)\st f(s,x) = \pi(s)\}
	\]is measurable and admits a measurable selector, i.e. in addition to being measurable, there is a measurable function $\xi:S\to X$ s.t. $\xi(s)\in \nu(s)$ i.e. $\pi(s)=f(s,\xi(s))$ for each $s\in S$.
\end{lemma}

\label{proof:maxent}
\maxent*
\begin{proof}
	
	For the existence and uniqueness of the maximum entropy distribution, let $\gamma\in \calQ(\nu,x)$ be any other measure, and denote by $m(x):=\diff{\mu^*_x}{\nu}$ and $n(x):=\diff{\gamma}{\nu}(x)$. Then
	\begin{align*}
	 	H_\nu(\gamma)&= -\int n(y)\log n(y)\nu(\dd y) \\
	 	&= -\int n(y)\log \frac{n(y)}{m(y)}\nu(\dd y) - \int n(y)\log m(y)\nu(\dd y)\\
	 	&=- \KL(\gamma\|\mu^*_x) - \int n(y)\log m(y)\nu(\dd y)\\
	 	&=- \KL(\gamma\|\mu^*_x) - \int n(y)[\ip{\lambda(x)}{k(x,y)} - A(x)]\nu(\dd y)\\
	 	&=- \KL(\gamma\|\mu^*_x) - \int [\ip{\lambda(x)}{k(x,y)} - A(x)]\gamma(\dd y)  \\
	 	&=- \KL(\gamma\|\mu^*_x) - [\ip{\lambda(x)}{f(x)} - A(x)] & \text{by linearity and }\gamma \in \calQ(\nu,x) \\
	 	&=- \KL(\gamma\|\mu^*_x) - \int [\ip{\lambda(x)}{k(x,y)} - A(x)]\mu^*_x(\dd y) & \text{by linearity and }\mu^*_x \in \calQ(\nu,x)\\
	 	&=- \KL(\gamma\|\mu^*_x) + H_{\nu}(\mu^*_x)\\
	 	&\leq H_\nu(\mu^*_x).
 	\end{align*} 		
 	since $\KL(\nu\|\mu^*_x)\geq 0$ and where $A(x) = \log \int \exp \ip{\lambda(x)}{k(x,y)}\nu(\dd y)$. This proves that $\mu^*_x$ is indeed the maximum entropy distribution, and is unique since any other maximum verifies $\KL(\gamma\|\mu^*_x) = 0$.

	The second statement is effectively a measurable selection theorem, for which we will apply Lemma~\ref{lem:himmel1} with:
	\begin{itemize}
	    \item $S = E$ and $A = \calP(E)$,
	    \item $R(x) = \calQ(\nu,x) \subset \calP(E)$ ($\nu$ is fixed throughout this theorem and its proof), and
	    \item $u(x,\mu) = H_\nu(\mu) = - \KL(\mu\|\nu) \in \R$ for $\mu \in \calQ(\nu,x)$.
	\end{itemize}
	
	First, $S = E$ is straightforwardly a Borel space. As far as $A$ is concerned, we are working on the 1-Wasserstein space $\calW_1=(\calP_1(E),\W_1)$, which is a complete metric space since $E\subset\R^d$ and $E$ is equipped with the Borel $\sigma$-algebra, hence $A = \calP(E)$ is also Borel.

	Second, by Lemma~\ref{lem:kl_usc}, $\mu \mapsto H_\nu(\mu) = -\KL(\mu\|\nu)$ is u.s.c. on $\calQ(\nu,x)$ for each $x\in E$. This implies in particular that $H_\nu^{-1}(]-\infty,a[)$ is open in the weak (i.e. 1-Wasserstein) topology for any $a$. Since the rays $]-\infty,a[$ generate the Borel sets, we get that $u$ is Borel-measurable.
    
	 In order to apply Lemma~\ref{lem:himmel1}, we need to prove that $R = x\mapsto \calQ(\nu,x)$ is measurable and compact-valued. This will conclude the proof as Lemma~\ref{lem:himmel1} guarantees that the selector, which corresponds to $x \to \mu_x^*$, is measurable.

	Let us thus prove that $x\mapsto \calQ(\nu,x)$ is measurable and compact-valued. Define $AC(\nu):=\{\mu\in \calW_1\st \mu \ll \nu\}$ and the mapping $\phi:E\times \calW_1\to E$ by $\phi(x,\mu)=\mu(k(x,\bdot))$; we can describe $\calQ(\nu,x)$ as the following:
	\[
		\calQ(\nu,x)=AC(\nu)\cap \{\mu\in \calW_1\st \phi(x,\mu)=f(x)\}.
	\]We will show that $AC(\nu)$ is measurable and closed in $\calW_1$, and that ${x \mapsto\{\mu\in \calW_1\st \phi(x,\mu)=f(x)\}}$ is a measurable, closed-valued correspondence, yielding the result. 

	For the first step, we have that $\{\mu = \Psi_X(\nu) \st X\in L^1(\R_{+};\nu)\}\subset AC(\nu)$. Conversely, $\mu\in AC(\nu)\implies \exists X\in L^1(\R_{+};\nu)$ s.t. $\mu(\dd x) = \Psi_X(\nu)$ by the Radon-Nikodym Theorem. Therefore, $AC(\nu) = \{\Psi_X(\nu)\st X\in L^1(\R_{+};\nu)\}$. Moreover, by Lemma~\ref{lem:psi_cont}, $X\mapsto \Psi_X(\nu)$ is a continuous map from $L^1(\R_+;\nu)\to \calW_1$, therefore $AC(\nu)$, which is the image of $L^1(\R_{+};\nu)$ under $\Psi_\bdot(\nu)$, is Borel measurable. Moreover, as $L^1(\R_{+};\nu)$ is closed, $AC(\nu)$ is also closed.

	For the second step, we apply Lemma~\ref{lem:filippov}. $\phi:E\times \calW_1\to E$ is a Caratheodory function. Indeed, $\forall \mu$, $x\mapsto\phi(x,\mu)$ is measurable by Fubini's Theorem, and $\forall x$, $\mu\mapsto \phi(x,\mu)$ is continuous in the weak (i.e. Wasserstein) topology. The correspondence from Lemma~\ref{lem:filippov}, $\vphi:E\to \calW_1$, is the constant correspondence $\vphi(x)=\calW_1$. The selector $\pi:E\to E$ is the function $f:E\to E$ which is measurable by assumption. Therefore, by Lemma~\ref{lem:filippov}, the level-set correspondence
	\[
		x\mapsto \{\mu \st \phi(x,\mu)=f(x)\}
	\]is measurable. Moreover, $\{f(x)\}$ is a closed set, and $\mu\mapsto \phi(x,\mu)$ is continuous so  $\phi(x,\bdot)^{-1}(\{f(x)\})$ is a closed set. This proves that $\calQ(\nu,x)$ is closed in $\calW_1$. Since $\calW_1$ is compact by compactness of $E$ \cite[ Ch 6]{villani2008optimal}, we have shown that $\calQ(\nu,x)$ is compact which concludes the proof.
\end{proof}

\label{proof:expmomproj}
\expmomproj*
\begin{proof} This proof is straightforward and follows the lines of \cite{koller2009probabilistic}, Theorem 8.6.
	Let $\theta_M=\mu(F)$ so that $\nu_{\theta_M}\in \calF$ is the measure which satisfies
	\[
		\nu_{\theta_M}(F) = \mu(F)
	\]and let $\nu_\theta$ be any other element of $\calF$. Then,
	\begin{align*}
		\KL(\mu\|\nu_{\theta_M}) - \KL(\mu\|\nu_\theta) &= \int \log \diff{\mu}{\nu_{\theta_M}}\dd \mu - \int \log \diff{\mu}{\nu_\theta}\dd \mu\\
		&= \int \log\diff{\nu_\theta}{\nu_{\theta_M}}\dd \mu\\
		&= \int \log \exp(\ip{\theta_M - \theta}{F(x)})\dd \mu(x) - \log\frac{\calZ(\theta_M)}{\calZ(\theta)}\\\
		&= \int \ip{\theta_M - \theta}{F(x)}\dd \mu(x) - \log\frac{\calZ(\theta_M)}{\calZ(\theta)}\\
		&= \ip{\theta_M-\theta}{\mu(F)} -  \log\frac{\calZ(\theta_M)}{\calZ(\theta)}\\
		&= \ip{\theta_M - \theta}{\nu_{\theta_M}(F)} - \log\frac{\calZ(\theta_M)}{\calZ(\theta)}\\
		&= - \KL(\nu_{\theta_M}\|\nu_\theta) \leq 0
	\end{align*}
	where $\calZ(\theta)=\exp(A(\theta))$ is the partition function.
\end{proof}

\begin{definition}[Measure Convolution]
	Let $\mu,\nu\in \calP(\R^d)$ be measures. Then we define the \defn{measure convolution} of $\mu$ and $\nu$, denoted $\mu*\nu$, by
	\[
		(\mu*\nu)(f) := \int f(x+y)\mu(\dd x)\nu(\dd y)~~~\forall f\in \calB_b(\R^d).
	\]
	\end{definition}

\begin{lemma}\label{prop:conv_stat}
	Let $\mu,\nu \in \calP(\R^d)$, and assume that $F:\R^d \to \R$ is a measurable, linear function, with the additional stipulation that $F\in L^1(\mu)\cap L^1(\nu)$. Then we have
	\[
		(\mu*\nu)(F) = \mu(F) + \nu(F).
	\]
\end{lemma}

\begin{proof}
We have:
    \begin{align*}
        (\mu*\nu)(F) &= \int F(x+y)\mu(\dd x)\nu(\dd y) = \int [F(x) + F(y)]\mu(\dd x)\nu(\dd y) \\ 
        &= \int F(x) \mu(\dd x)\nu(\dd y) + \int F(y)\mu(\dd x)\nu(\dd y) = \mu(F) + \nu(F),
    \end{align*}
which concludes the proof.

\end{proof}

\begin{lemma}\label{lem:conv_limit}
	Define $\rho_n:=\calN(0,\sigma^2_nI)$ be a multivariate Gaussian on $\R^d$, $\sigma_n\to 0$, and let $\mu\in\calP_1(\R^d)$. Then $\rho_n*\mu\in \calP_1(\R^d)$ and  $\W_1(\rho_n*\mu, \mu)\to 0$ as $n\to \infty$.
\end{lemma}
\begin{proof} Let $x_0 \in \mathbb{R}^d$.  A measure $\nu$ has finite 1st moment iif  $\int \|x_0 - x\| \nu(\dd x) < \infty$. 
Firstly, 
    \begin{align*}
        \int \|x_0 - x\| \rho_n*\mu(\dd x)&= \iint \|x_0 - (x + y)\|\rho_n(\dd x) \mu(\dd y) \\
        &\leq \iint (\|x_0/2 - x\| + \|x_0/2 - y\|)\rho_n(\dd x) \mu(\dd y)\\
        &\leq \int \|x_0/2 - x\|\rho_n(\dd x) + \int \|x_0/2 - y\|\mu(\dd y)<\infty
    \end{align*}
    so that $\rho_n*\mu\in \calP_1(\R^d)$. Now let $f\in Lip(1)$, then
	\begin{align*}
		|(\rho_n*\mu)(f) - \mu(f)| &= \left|\iint[f(x+y) - f(y)]\rho_n(\dd x)\mu(\dd y)\right|\\
		&\leq \iint |f(x+y) - f(y)|\rho_n(\dd x)\mu(\dd y)\\
		&\leq \int \|x+y-y\|\mu(\dd y)\rho_n(\dd x)\\
		&= \int \|x\| \rho_n(\dd x) \to \int \|x\| \delta_0(\dd x) = 0
	\end{align*}
	as $n\to \infty$. Since the upper bound is independent of $f$, we are done.
\end{proof}

\label{proof:deltamomproj}
\deltamomproj*
\begin{proof}
	Let $f:=\mu(F)$. Below, we extend $\mu\in \calP_1(E)$ to be a (compactly supported) element of $\calP_1(\R^d)$ by simply setting ${\mu(A)=\mu(A\cap E)~\forall A\in \calB(\R^d)}$, and trivially extend $F$ to $\R^d$.
	\begin{enumerate}
		\item We have that $\rho_n*\mu(F)=\mu(F)~\forall n$ by Lemma~\ref{prop:conv_stat} since $F(x)=x$ is linear and $\rho_n(F) = 0$. Therefore, $\Pi_\calF(\rho_n*\mu)=\Pi_\calF(\mu), \forall n$.
		\item Let $\calF_n$ be as above. 
		By the previous step, we have 
		\[
			\Pi_{\calF_n}(\rho_n*\mu)  = \Pi_{\calF_n}(\mu) = (\rho_n*\nu)_f = \rho_n*\nu_f.
		\]Moreover, since $\rho_n*\mu\sim \nu~\forall \nu\in \calF_n$ and $\calF_n$ is an exponential family, by Proposition~\ref{prop:expmomproj}
		\[
			\Pi_{\calF_n}(\rho_n*\mu) = \arg\min_{\nu \in \calF_n}\KL(\rho_n*\mu\|\nu).
		\]
		\item Finally, $\nu_f := \Pi_{\calF}(\mu)$ is equal to the limit
		\[
			\nu_f \overset{(1)}{=}  \lim_{n\to \infty} \rho_n*\nu_f  \overset{(2)}{=} \lim_{n\to \infty}\Pi_{\calF_n}(\rho_n*\mu) \overset{(3)}{=}  \lim_{n\to \infty}\arg\min_{\nu\in \calF_n}\KL(\rho_n*\mu\|\nu)
		\]where (1) is due to Lemma~\ref{lem:conv_limit}, and (2), (3) are due to step 2. Hence we can think of $\nu_f$ as the limit of the minimizers in the 1-Wasserstein sense.
	\end{enumerate}
\end{proof}

\section{Proofs From Section~\ref{sec:stability}}\label{app:stab_proofs}

\label{proof:psicontract}
\psicontract*
\begin{proof}
For the first inequality, let $f$ be any 1-Lipschitz function, and $x,y \in E$. We have:
\begin{align*}
    |\Psi_{G(x,\bdot)}(\mu)(f) - \Psi_{G(y,\bdot)}(\mu)(f)| &= \left|\int \frac{G(x,z) f(z) }{\mu(G(x, \bdot))} - \frac{G(y,z) f(z)}{\mu(G(y, \bdot))} \mu(dz)\right| \\
    &\leq \left|\int \frac{G(x,z) f(z) }{\mu(G(x, \bdot))} - \frac{G(x,z) f(z)}{\mu(G(y, \bdot))} \mu(dz)\right| \\
    & \hspace{2cm} + \left|\int \frac{G(x,z) f(z)}{\mu(G(y, \bdot))} - \frac{G(y,z) f(z)}{\mu(G(y, \bdot))} \mu(dz)\right|.
\end{align*}
Let us bound the first term:
\begin{align*}
    \left|\int \frac{G(x,z) f(z)}{\mu(G(x, \bdot))} - \frac{G(x,z) f(z)}{\mu(G(y, \bdot))} \mu(dz)\right| &\leq \frac{|\mu(G(x, \bdot)) - \mu(G(y, \bdot))|}{\mu(G(x, \bdot))\mu(G(y, \bdot))} \int G(x,z) |f(z)| \mu(dz) \\
    &\leq \frac{\int |G(x, z) - G(y, z)| \mu(dz)}{\mu(G(x, \bdot))\epsilon(G)} \mu(G(x, \bdot)) \|f\|_{\infty} \\
    &\leq \frac{\int |G(x, z) - G(y, z)| \mu(dz) \|f\|_{\infty}}{\epsilon(G)d(x,y)}d(x,y) \\
    &\leq \frac{\|G\|_{Lip,\infty} \|f\|_{\infty}}{\epsilon(G)}d(x,y).
\end{align*}
Let us now bound the second term:
\begin{align*}
    \left|\int \frac{G(x,z) f(z)}{\mu(G(y, \bdot))} - \frac{G(y,z) f(z)}{\mu(G(y, \bdot))} \mu(dz)\right| &\leq \frac{\|f\|_{\infty}}{\epsilon(G)} \int |G(x,z) - G(y,z)| \mu(dz) \\
    &\leq \frac{\|G\|_{Lip,\infty} \|f\|_{\infty}}{\epsilon(G)}d(x,y).
\end{align*}
Using the fact that $\Psi_{G(x,\bdot)}(\mu)(\bar{f}) = \Psi_{G(y,\bdot)}(\mu)(\bar{f})$ for any constant function $\bar{f}$, we can subtract from $f$ any constant without changing the value of $|\Psi_{G(x,\bdot)}(\mu)(f) - \Psi_{G(y,\bdot)}(\mu)(f)|$. This allows us to assume without loss of generality that $\|f\|_{\infty} \leq \diam(E)$ (picking an arbitrary $x \in E$, we have $\forall y \in E$, $|f(y) - f(x)| \leq |y - x| \|f\|_{Lip} \leq \diam(E)$). Combining everything, we get:
\[
|\Psi_{G(x,\bdot)}(\mu)(f) - \Psi_{G(y,\bdot)}(\mu)(f)| \leq 2\frac{\|G\|_{Lip,\infty} \diam(E)}{\epsilon(G)}d(x,y).
\]
Taking the supremum over 1-Lipschitz functions $f$ concludes the first part of the proof.

Let us now prove the second inequality. Similarly, let $f$ be any 1-Lipschitz function, and $\mu,\nu$ two compactly supported distributions on $(E,\calE)$. We use the notation $G(z):=G(x,z)$ for this part because $x$ is fixed. We have:
\begin{align*}
    |\Psi_{G}(\mu)(f) - \Psi_{G}(\nu)(f)| &= \left|\int \frac{G(z) f(z)}{\mu(G)} \mu(dz) - \int \frac{G(z) f(z)}{\nu(G)} \nu(dz)\right| \\
    &\leq \left|\int \frac{G(z) f(z)}{\mu(G)} \mu(dz) - \int \frac{G(z) f(z)}{\nu(G)} \mu(dz)\right| \\
    & \hspace{2cm} + \left|\int \frac{G(z) f(z)}{\nu(G)} \mu(dz) -\int \frac{G(z) f(z)}{\nu(G)} \nu(dz)\right|.
\end{align*}
Let us bound the first term:
\begin{align*}
    \left|\int \left(\frac{G(z) f(z)}{\mu(G)} - \frac{G(z) f(z)}{\nu(G)}\right) \mu(dz)\right| &\leq \frac{|\mu(G) - \nu(G)|}{\mu(G)\nu(G)} \int G(z) |f(z)| \mu(dz) \\
    &\leq \frac{\|G\|_{Lip} \W_1(\mu,\nu)}{\mu(G)\epsilon(G)} \mu(G) \|f\|_{\infty} \\
    &\leq \frac{\|G\|_{Lip} \|f\|_{\infty}}{\epsilon(G)} \W_1(\mu,\nu).
\end{align*}
Let us now bound the second term:
\begin{align*}
    \left|\int \frac{G(z) f(z)}{\nu(G)} \mu(dz) -\int \frac{G(z) f(z)}{\nu(G)} \nu(dz)\right| &\leq \frac{\|f\|_\infty}{\nu(G)}\left|\int G(z) \mu(dz) -\int G(z) \nu(dz)\right|\\
    &\leq \frac{\|G\|_{Lip} \|f\|_{\infty}}{\epsilon(G)}\W_1(\mu,\nu).
\end{align*}
Using the same reasoning as above, we can assume without loss of generality that $\|f\|_{\infty} \leq \diam(E)$, which gives:
\[
|\Psi_{G}(\mu)(f) - \Psi_{G}(\nu)(f)| \leq 2 \frac{\|G\|_{Lip} \diam(E)}{\epsilon(G)}\W_1(\mu,\nu).
\]
Taking the supremum over all 1-Lipschitz functions $f$ concludes the proof.
\end{proof}

\label{proof:picontract}
\picontract*
\begin{proof}
	Denote by $\pi_i:E\to \R$ the canonical projection onto the $i$-th coordinate of $E\subset \R^d$, and let $x_i := \pi_i(x)$. Then
	\begin{align*}
 		\W_1(\Pi(\mu), \Pi(\nu)) &= \W_1(\delta_{\mu(F)}, \delta_{\nu(F)})\\
	 	&=\|\mu(F) - \nu(F)\|_1\\
	 	&= \sum_{i=1}^d |\mu(F)_i - \nu(F)_i| \\
 		&= \sum_{i=1}^d |\mu(\pi_i\circ F) - \nu(\pi_i\circ F)|\\
 		&\leq d \cdot \max_{i=1,\dots,d}\{|\mu(\pi_i\circ F) - \nu(\pi_i\circ F)|\}\\
 		&\leq d\cdot \sup_{f\in Lip(1)}|\mu(f) - \nu(f)|\\
 		&= d\cdot \W_1(\mu,\nu)
 	\end{align*}
    since $\pi_i\circ F\in Lip(1)$ for $i=1,\dots,d$ (we recall that here, $F$ is the identity).
\end{proof}

\label{proof:lookupcontract}
\lookupcontract*
\begin{proof}
	\begin{align*}
		\W_1(\mu L,\gamma L) &= \sup_{f\in Lip(1)}\left|\int f(x)\mu L(\dd x) - \int f(y)\gamma L(\dd y)\right|\\
		&= \sup_{f\in Lip(1)} \left|\int f(x)\int \mu(\dd z)L(z,\dd x) - \int f(y)\int \gamma(\dd z) L(z,\dd y)\right|\\
		&= \sup_{f\in Lip(1)} \left|\iint  f(x)L(z,\dd x)\mu(\dd z) - \iint f(y)L(z,\dd y)\gamma(\dd z)\right|\\
		&= \sup_{f\in Lip(1)} \left|\iint  f(x)\delta_{\ell(z)}(\dd x)\mu(\dd z) - \iint f(y)\delta_{\ell(z)}(\dd y)\gamma(\dd z)\right|\\
		&= \sup_{f\in Lip(1)} \left| \int f\circ \ell(z)\mu(\dd z) - \int f\circ \ell(z) \gamma(\dd z)\right|.
	\end{align*}
	Then since $\|f\|_{Lip} = 1$, we have $\|f\circ \ell\|_{Lip} \leq \|f\|_{Lip} \|\ell\|_{Lip} = K_\ell$. Hence, by our earlier estimation techniques:
\begin{align*}
    	\W_1(\mu L,\gamma L) &= \sup_{f\in Lip(1)} \left| \int f\circ \ell(\dd z)\mu(\dd z) - \int f\circ \ell(z) \gamma(\dd z)\right| \\ 
		&\leq K_\ell  \sup_{g\in Lip(1)} \left| \int g(\dd z)\mu(\dd z) - \int g(z) \gamma(\dd z)\right| =K_\ell \W_1(\mu,\gamma),
\end{align*}
which concludes the proof.
\end{proof}

\begin{lemma}\label{lem:tau}
	\begin{enumerate}
		\item Suppose that $\Phi,\Gamma:\calP(E)\to \calP(E)$ are (possibly nonlinear) mappings. Then
		\[
			\tau(\Phi\circ\Gamma) \leq \tau(\Phi)\tau(\Gamma).
		\]
		\item Suppose $K:E\times \calE\to [0,1]$ is an integral kernel. Then
		\[
			\tau(K) = \sup_{x\neq y}\frac{\W_1(K(x,\bdot), K(y,\bdot))}{d(x,y)}.
		\]
		\item Suppose $K_1,K_2:E\times \calE\to [0,1]$ are two integral kernels and $\nu \in \mathcal{P}(E)$. Then:
		\[
			\W_1(\nu K_1, \nu K_2) \leq \int \nu(dx) \W_1(K_1(x,\bdot), K_2(x,\bdot)).
		\]
	\end{enumerate}
\end{lemma}
\begin{proof}
	\begin{enumerate}
		\item This is a standard result on Lipschitz constants. We include it for completeness:
		\begin{align*}
			\tau(\Phi\circ\Gamma) &= \sup_{\mu\neq \nu}\frac{\W_1(\Phi\circ\Gamma(\mu), \Phi\circ\Gamma(\nu))}{\W_1(\mu,\nu)}\\
			&= \sup_{\mu\neq \nu}\frac{\W_1(\Phi\circ\Gamma(\mu), \Phi\circ\Gamma(\nu))}{\W_1(\Gamma(\mu),\Gamma(\nu))}\frac{\W_1(\Gamma(\mu), \Gamma(\nu))}{\W_1(\mu,\nu)}\\
			&\leq \sup_{\eta\neq \gamma}\frac{\W_1(\Phi(\eta), \Phi(\gamma))}{\W_1(\eta,\gamma)}\cdot \sup_{\mu\neq \nu}\frac{\W_1(\Gamma(\mu), \Gamma(\nu))}{\W_1(\mu,\nu)}\\
			&= \tau(\Phi)\tau(\Gamma).
		\end{align*}
		\item Since $\W_1(\delta_x, \delta_y) = d(x,y)$ and $\delta_xK = K(x, \bdot)$ we have:
		\[
			\sup_{x\neq y}\frac{\W_1(K(x,\bdot), K(y,\bdot))}{d(x,y)}
			= \sup_{\delta_x\neq \delta_y}\frac{\W_1(\delta_x K, \delta_y K)}{\W_1(\delta_x,\delta_y)}\leq \sup_{\mu\neq \nu}\frac{\W_1(\mu K, \nu K)}{\W_1(\mu,\nu)}.
		\]For the reverse inequality,
		\begin{align*}
			\W_1(\mu K, \nu K)& = \sup_{f\in Lip(1)} |\mu K(f) - \nu K(f)|\\
			&= \sup_{f\in Lip(1)} |\mu(Kf) - \nu(Kf)|\\
			&\leq \sup_{f\in Lip(1)} \|Kf\|_{Lip(1)} \cdot \sup_{g\in Lip(1)} |\mu(g) - \nu(g)|\\
			&\leq \sup_{f\in Lip(1)}\|Kf\|_{Lip(1)}\cdot \W_1(\mu,\nu)
		\end{align*}
		and 
		\begin{align*}
			\sup_{f\in Lip(1)}\|Kf\|_{Lip(1)} &= \sup_{f\in Lip(1)}\sup_{x\neq y} \frac{\int K(x,\dd z)f(z)- \int K(y,\dd z)f(z)}{d(x,y)}\\
			&= \sup_{f\in Lip(1)}\sup_{x\neq y} \frac{\int [K(x,\dd z) - K(y,\dd z)]f(z)}{d(x,y)}\\
			&= \sup_{x\neq y}\frac{\W_1(K(x,\bdot),K(y,\bdot))}{d(x,y)}.
		\end{align*}
		Dividing by $\W_1(\mu,\nu)$ gives us the reverse inequality and concludes the proof.
		\item By definition, we have:
		\begin{align*}
		    \W_1(\nu K_1, \nu K_2) &= \sup_{f\in Lip(1)} |\nu K_1(f) - \nu K_1(f)| \\
		    &= \sup_{f\in Lip(1)} \left|\iint \nu(dx) K_1(x,dy) f(y) - \iint \nu(dx) K_2(x,dy) f(y)\right| \\
		    &\leq \sup_{f\in Lip(1)} \int \nu(dx) \left|\int K_1(x,dy) f(y) - K_2(x,dy) f(y)\right| \\
		    &\leq \int \nu(dx) \W_1(K_1(x,\bdot), K_2(x,\bdot)).
		\end{align*}
	\end{enumerate}
\end{proof}

\noindent
Using Propositions~\ref{prop:psi_contract}, \ref{prop:pi_contract} and \ref{prop:lookup_contract} and Lemma~\ref{lem:tau}, we can prove Theorem~\ref{thm:contraction}.
\contraction*
\begin{proof}
    We want to bound $\displaystyle{\sup_{\mu\neq \nu}\frac{\W_1(\mu A_\mu, \nu A_\nu)}{\W_1(\mu,\nu)}}$. Let $\mu \neq \nu \in \calP(E)$, we have:
    \begin{align*}
        \frac{\W_1(\mu A_\mu, \nu A_\nu)}{\W_1(\mu,\nu)} &\leq \frac{\W_1(\mu A_\mu, \nu A_\mu)}{\W_1(\mu,\nu)} + \frac{\W_1(\nu A_\mu, \nu A_\nu)}{\W_1(\mu,\nu)}
    \end{align*}
    
Let us start with the first term:
    \begin{align*}
        \frac{\W_1(\mu A_\mu, \nu A_\mu)}{\W_1(\mu,\nu)}
        &\leq \frac{\W_1(\mu\Pi[\Psi_{G(\bdot,\bdot)}(\mu)L], \nu\Pi[\Psi_{G(\bdot,\bdot)}(\mu)L])}{\W_1(\mu,\nu)} \\
        &\leq \sup_{x\neq y}\frac{\W_1(\Pi[\Psi_{G(x,\bdot)}(\mu)L], \Pi[\Psi_{G(y,\bdot)}(\mu)L])}{d(x,y)} \\
        &\leq \tau_1(\Pi) \tau_1(L) \sup_{x\neq y}\frac{\W_1(\Psi_{G(x,\bdot)}(\mu), \Psi_{G(y,\bdot)}(\mu))}{d(x,y)}\\
        &\leq \tau_1(\Pi) \tau_1(L) \frac{2\|G\|_{Lip,\infty}\diam(E)}{\epsilon(G)}, 
    \end{align*}
where we used Lemma~\ref{lem:tau} for the second and third lines, and Propositions~\ref{prop:psi_contract}, \ref{prop:pi_contract} and \ref{prop:lookup_contract} for the third and last. As for the second term, we have:
    \begin{align*}
        \W_1(\nu A_\mu, \nu A_\nu)
        &=\W_1(\nu\Pi[\Psi_{G}(\mu)L], \nu\Pi[\Psi_{G}(\nu)L]) \\
        &\leq \int \nu(dx) \W_1(\Pi[\Psi_{G(x,\bdot)}(\mu)L], \Pi[\Psi_{G(x,\bdot)}(\nu)L]) \\
        &\leq \tau_1(\Pi) \tau_1(L) \int \nu(dx) \W_1(\Psi_{G(x,\bdot)}(\mu), \Psi_{G(x,\bdot)}(\nu)) \\
        &\leq \tau_1(\Pi) \tau_1(L) \int \nu(dx) \frac{2\|G(x,\bdot)\|_{Lip}\diam(E)}{\epsilon(G)} \W_1(\mu,\nu) \\
        &\leq \tau_1(\Pi) \tau_1(L) \frac{2 \|G\|_{\infty,Lip}\diam(E)}{\epsilon(G)}\W_1(\mu,\nu)
    \end{align*}
    where we also used Lemma~\ref{lem:tau} for the second and third lines, and Propositions~\ref{prop:psi_contract}, \ref{prop:pi_contract} and \ref{prop:lookup_contract} for the third and last.
\end{proof}

\section{Proofs From Section~\ref{subsec:unbounded}}\label{app:unbounded}

\begin{lemma}\label{lem:local_lip}
    For any $f:\R^d\to \R$, we have 
    \begin{equation}
        \|f\|_{Lip} = \sup_{x \neq y, \|x-y\| \leq 1} \frac{|f(x) - f(y)|}{\|x - y\|}.
    \end{equation}
\end{lemma}
\begin{proof}
    Let $x \neq y$ and $L := \sup_{x \neq y, \|x-y\| \leq 1} \frac{|f(x) - f(y)|}{\|x - y\|}\leq \infty$. First, assume $\|f\|_{Lip},L<\infty$. It is clear that $L\leq \|f\|_{Lip}$ since $\{x\neq y,\|x-y\|\leq 1\}\subset\{x\neq y\}$. For the reverse inequality, we split the segment $[x,y]$ into the minimum number of chunks of lengths smaller than 1: $x = z_1 \rightarrow z_2 \rightarrow \cdots \rightarrow z_k = y$ (in particular, if $\|x - y \| \leq 1$ then $z_2 = y$). Then
    \begin{align*}
        |f(x) - f(y)| &\leq \sum_{1 \leq i \leq k-1} |f(z_i) - f(z_{i+1})| \\
        &\leq L \sum_{1 \leq i \leq k-1} \|z_i - z_{i+1}\| = L \|x - y\|.
    \end{align*}
    which gives $\|f\|_{Lip} \leq L$ so $L=\|f\|_{Lip}$. Now if $\|f\|_{Lip}=\infty$ but $L<\infty$, by applying the above argument we can obtain a contradiction. Finally, it suffices to note that the case where $\|f\|_{Lip}<\infty$ but $L=\infty$ is impossible since $\|f\|_{Lip}\geq L$.
\end{proof}

\begin{lemma}\label{lem:ratio}
    For any $n$ and $(z_1,\cdots,z_n) \in \mathbb{R}^n_+$:
    \begin{equation}
    f(z_1,\cdots,z_n) := \frac{\sum^{n}_{i=1} z_i e^{-z_i^2} }{1 + \sum^{n}_{i=1} e^{-z_i^2}} \leq \sqrt{\ln{n} + \frac{1}{2e}}.
    \end{equation}
\end{lemma}
\begin{proof}
    $f$ is clearly bounded on $\mathbb{R}^n_+$ ($z_i e^{-z_i^2} \to 0$ when $z_i \to \infty$).
    Let us now compute the partial derivatives of $f$. For a given $z_i$:
    \begin{align*}
        \frac{\partial f}{\partial z_i} = \frac{e^{-z_i^2}}{1 + \sum^{n}_{k=1} e^{-z_k^2}} [1 - 2 z_i^2 + 2 z_i f(z_1,\cdots,z_n)].
    \end{align*}
    There is only one positive solution of $1 - 2 z_i^2 + 2 z_i f^* = 0$, meaning that $f$ reaches its maximum when all its coordinates are equal. We thus only need to study:
    \begin{equation}
        g(x) := \frac{n x e^{-x^2}}{1 + n e^{-x^2}} = \frac{x e^{\ln{n} - x^2}}{1 + e^{\ln{n} - x^2}}.
    \end{equation}
    The change of variable $y = \ln{n} - x^2$ gives $g(y) = \frac{\sqrt{\ln{n} - y} e^{y}}{1 + e^{y}} \leq \frac{\sqrt{\ln{n} - y}}{1 + e^{-y}}$ with $y \in ]-\infty,\ln{n}]$. 
    
    On $[0,\ln{n}]$, we clearly have $g(y) \leq \sqrt{\ln{n}}$. Let us consider $y \in ]-\infty,0]$. We get $g^2(y) = \frac{\ln{n} - y}{(1 + e^{-y)^2}} \leq \frac{\ln{n} - y}{e^{-2y}} \leq \ln{n} + \frac{1}{2e}$ with since $(2e)^{-1}$ is the maximum of of $z e^{-2z}$ on $\mathbb{R}_+$. This concludes the proof.
\end{proof}

\begin{lemma}\label{lem:tensor}
	Let $\mu_1,\mu_2,\nu_1,\nu_2\in \calW_1(\R^d)$. Then
	\[
		\W_1(\mu_1\otimes\mu_2, \nu_1\otimes \nu_2)\leq \W_1(\mu_1,\nu_1) + \W_1(\mu_2, \nu_2)
	\]
\end{lemma}
\begin{proof}
Let $\gamma_1\in \calC(\mu_1,\nu_1), \gamma_2\in \calC(\mu_2, \nu_2)$ be optimal for $c(x,y)=\|x-y\|_1$. Note that $\gamma_1\otimes\gamma_2\in \calC(\mu_1\otimes\mu_2, \nu_1\otimes \nu_2)$, i.e. $\gamma_1\otimes\gamma_2$ is a transfer plan with the correct marginals, by considering 
\begin{align*}
	\int_{X\times X}\dd \gamma_1\otimes\gamma_2(x_1, x_2,y_1,y_y) &= \int_{X\times X} \dd \gamma_1(x_1, y_1)\dd\gamma(x_2, y_2)\\
	&= \int_X \dd\gamma_1( x_1, y_1)\int_X \dd\gamma_2( x_2, y_2) \\
	&= \nu_1(\dd y_1)\nu_2(\dd y_2) = \dd \nu_1\otimes\nu_2(y_1,y_2)
\end{align*}
	and same for the other marginals.

	Thus we have
	\begin{align*}
		\W_1(\mu_1\otimes\mu_2, \nu_1\otimes \nu_2) &= \inf_{\gamma\in \calC(\mu_1\otimes\mu_2, \nu_1\otimes \nu_2)}\int \|(x_1,x_2) - (y_1,y_2)\|\dd\gamma(x_1,x_2, y_1,y_2)\\
		&= \inf_{\gamma\in \calC(\mu_1\otimes\mu_2, \nu_1\otimes \nu_2)}\int (\|x_1 - y_1\| +  \|y_1,y_2\|)\dd\gamma(x_1,x_2, y_1,y_2)\\
		&= \inf_{\gamma\in \calC(\mu_1\otimes\mu_2, \nu_1\otimes \nu_2)}\int \|x_1 - y_1\|\dd\gamma(x_1,x_2, y_1,y_2) +\cdots \\
		&\hspace{1cm}\cdots + \inf_{\gamma\in \calC(\mu_1\otimes\mu_2, \nu_1\otimes \nu_2)}\int \|x_2 - y_2\|\dd\gamma(x_1,x_2, y_1,y_2)\\
		& \leq \int \|x_1 - y_1\|\dd\gamma_1\otimes \gamma_2(x_1,x_2, y_1,y_2) + \int \|x_2 - y_2\|\dd\gamma_1\otimes\gamma_2(x_1,x_2, y_1,y_2)\\
		&= \int \|x_1 - y_1\|\dd\gamma_1(x_1, y_1) + \int \|x_2 - y_2\|\dd\gamma_2(x_2, y_2)\\
		&= \W_1(\mu_1,\nu_1) + \W_1(\mu_2, \nu_2)
	\end{align*}

\end{proof}

\psiContractUnbounded*
\begin{proof}
We use the Kantorovich formulation of $\W_1$. Let $f$ be a function with $\|f\|_{Lip} \leq 1$. Using the same kind of technique as in Section~\ref{app:stab_proofs}, we can assume without loss of generality that $f(y) = 0$. For simplicity, we write $G(x,\bdot) = G_x$. We wish to upper-bound the quantity $|\Psi_{G_x}(\mu)(f) - \Psi_{G_y}(\nu)(f)|$.

Because $\Psi_{G_x}$ and $\Psi_{G_y}$ are homonegeous in their measure argument, and for the sake of simplicity, we write $\mu = \sum_i \delta_{x_i}$ $\nu = \sum_i \delta_{y_i}$ (which is equivalent to simplifying by $1/N$ in e.g. the numerator and denominator of $\Psi_{G_x}$). This guarantees in particular that $\mu(G_x) \geq 1$ and $\nu(G_y) \geq 1$ ($x$ and $y$ are in $\supp{\mu}$ and $\supp{\nu}$ resp.) and equivalently that $1 / \mu(G_x) \leq 1$ and  $1 / \nu(G_y) \leq 1$.

Then:
    \begin{align}
	|\Psi_{G_x}(\mu)(f) - \Psi_{G_y}(\nu)(f)| &= \frac{1}{\mu(G_x)\nu(G_y)}|\nu(G_y)\mu(G_xf) - \mu(G_x)\nu(G_yf)| \nonumber \\
	&= \frac{1}{\mu(G_x)\nu(G_y)}|\nu(G_y)\mu(G_xf) - \nu(G_y)\nu(G_yf) + \nu(G_y)\nu(G_yf) - \mu(G_x)\nu(G_yf)| \nonumber \\
	&\leq  \frac{\nu(G_y)}{\mu(G_x)\nu(G_y)}|\mu(G_xf) - \nu(G_yf)| + \frac{\nu(G_yf)}{\mu(G_x)\nu(G_y)}|\nu(G_y) - \mu(G_x)|.\label{eq:up}
\end{align}
We start by bounding the second term of~\eqref{eq:up}. We have:
\begin{align*}
	\frac{\nu(G_yf)}{\mu(G_x)\nu(G_y)}|\nu(G_y) - \mu(G_x)| &= \frac{\nu(G_yf)}{\mu(G_x)\nu(G_y)}|(\delta_x\otimes \mu)(G) - (\delta_y\otimes \nu)(G)|\\
	&\leq \frac{\nu(G_yf)}{\mu(G_x)\nu(G_y)}\|G\|_{Lip}\W_1(\delta_x\otimes \mu,\delta_y\otimes \nu).
\end{align*}
Here, $\delta_x\otimes \mu$ denotes the product of the two measures on $E \times E$. Since $f(y) = 0$, we see that $f(z) \leq  f(y) + \|f\|_{Lip}\|y-z\|_1\leq \|y-z\|_1$. This gives:
\begin{align*}
	\frac{\nu(G_y f)}{\nu(G_y)} = \frac{\int G_y(z)f(z) \nu(\dd z)}{\int G_y(z) \nu(\dd z)} &\leq \frac{\int G_y(z) \|y - z\|_1 \nu(\dd z)}{\int G_y(z) \nu(\dd z)} \\ 
	&\leq \frac{\sum^N_{i=1} G(y,y_i) \|y - y_i\|_1}{\sum^N_{i=1} G(y,y_i) } \leq \sqrt{d} \frac{\sum^N_{i=1} e^{-\|y-y_i\|_2^2} \|y - y_i\|_2}{\sum^N_{i=1} e^{-\|y-y_i\|_2^2}},
\end{align*}
where we applied Cauchy-Schwartz for the last inequality.  Since $y = y_i$ for a given $i$, we are interested in the quantity $\frac{\sum^{N-1}_{i=1} z_i e^{-z_i^2} }{1 + \sum^{N-1}_{i=1} e^{-z_i^2}}$ for arbitrary $z_i \geq 0$. Applying Lemma~\ref{lem:ratio} with $n = N-1$ gives an upper-bound of $\sqrt{\ln{N} + \frac{1}{2e}}$.

Let us now consider the first term of~\eqref{eq:up}:
\begin{align*}
	\frac{\nu(G_y)}{\mu(G_x)\nu(G_y)}|\mu(G_xf) - \nu(G_yf)|&= \frac{1}{\mu(G_x)}|\mu(G_xf) - \nu(G_yf)|\\
	&\leq  \frac{1}{\mu(G_x)}\|Gf\|_{Lip}\W_1(\delta_x\otimes \mu,\delta_y\otimes \nu).
\end{align*}
To estimate $\|Gf\|_{Lip}$ we have
\[
	\|Gf\|_{Lip} = \sup_{(x,w)\neq (y,z)} \frac{|G(x,w)f(w) - G(y,z)f(z)|}{\|(x,w) - (y,z)\|_1}
\]
where additionally, we can assume that $\|(x,w) - (y,z)\| \leq 1$ (see Lemma~\ref{lem:local_lip}). We have:
\begin{align*}
	|G(x,w)f(w) - G(y,z)f(z)|&= |G(x,w)f(w) - G(x,w)f(z) + G(x,w)f(z) - G(y,z)f(z)|\\
	&\leq |G(x,w)||f(w)- f(z)| + |f(z)||G(x,w) - G(y,z)|. 
\end{align*}
For the first term, we see that
\begin{align*}
	|G(x,w)||f(w) - f(z)| &\leq \|G\|_{\infty,\infty} \|f\|_{Lip}d(w,z)\\
	&\leq \|G\|_{\infty,\infty} \|f\|_{Lip}(d(w,z) + d(x,y)).
\end{align*}
For the second term, we have

\begin{align*}
    |f(z)||G(x,w) - G(y,z)| &\leq \|y - z\|_1 |G(x,w) - G(y,z)| \\
    &\leq \|y - z\|_1 \|\nabla G(t_1, t_2))\|_\infty \|(x,w) - (y,z)\|_1,
\end{align*}
for $t_1$ in the segment $[x,y]$ and $t_2$ in the segment $[w,z]$ (this follows directly from the mean value theorem, note that the gradient is taken with respect to both variables). We used $f(y) = 0$ and $f(z) \leq  f(y) + \|f\|_{Lip}\|y-z\|_1 = \|y-z\|_1$ in the first line.

\noindent
In the Gaussian case:
\begin{align*}
    \|y - z\|_1 \|\nabla G(t_1, t_2))\|_\infty &\leq (\|y - t_1\|_1  + \|t_1 - t_2\|_1 + \|t_2 - z\|_1) 2 \|t_1 - t_2\|_\infty e^{- \|t_1 - t_2\|^2_2} \\
    &\leq 2 (2 + \|t_1 - t_2\|_1) \|t_1 - t_2\|_\infty e^{- \|t_1 - t_2\|^2_2},
\end{align*}
where we used the fact that $\|y - t_1\|_1 \leq 1$ and $\|t_2 - z\|_1 \leq 1$ ($t_1$ is in the $[x,y]$ segment and $\|x - y\|_1 \leq 1$ by assumption). That upper bound is uniformly bounded with respect to $t_1$ and $t_2$, we let $C$ denote that constant. A loose upper-bound on $C$ is $\sqrt{d}+2$ (which we use in the statement of the proposition).

To conclude, it suffices to note that  by Lemma~\ref{lem:tensor} we have
\[
	\W_1(\delta_x \otimes \mu , \delta_y\otimes \nu)\leq \W_1(\delta_x,\delta_y) + \W_1(\mu,\nu).
\]
\end{proof}

\tauAunbounded*
\begin{proof}
    Firstly, using Proposition~\ref{prop:measure_attn}, we know that  $\mu\mbf{A}_{\mu}$ is another empirical measure concentrated on $\{\attn(x_i,X,X)\}$, similarly, $\nu\mbf{A}_{\nu}$ is concentrated on $\{\attn(y_i,Y,Y)\}$. This fact allows us to use the following result from~\cite{santambrogio2015optimal}~Equation 6.2
    \begin{align*}
    	\W_1(\mu,\nu) &= \min\left\{\sum_{i,j}\gamma_{ij}d(x_i,y_j) \st \gamma_{i,j}\geq 0,~\sum_i\gamma_{ij}=\frac{1}{M},~\sum_j \gamma_{ij}=\frac{1}{N}\right\},
    \end{align*}
    Applied to $\W_1(\mu \mbf{A}_{\mu},\nu \mbf{A}_{\nu})$, it gives
	\begin{align*}
		\W_1(\mu\mbf{A}_{\mu},\nu\mbf{A}_\nu)&= \min\Big\{\sum_{i,j}\gamma_{ij}d(\attn(x_i,X,X),\attn(y_j,Y,Y)) \st \\
		&\hspace{4cm}\gamma_{i,j}\geq 0,~\sum_i\gamma_{ij}=\frac{1}{M},~\sum_j \gamma_{ij}=\frac{1}{N}\Big\} \\
    	&= \min\Big\{\sum_{i,j}\gamma_{ij}\W_1(\mbf{A}_{\mu}(x_i,\bdot),\mbf{A}_{\nu}(y_i,\bdot)) \st \\
		&\hspace{4cm} \gamma_{i,j}\geq 0,~\sum_i\gamma_{ij}=\frac{1}{M},~\sum_j \gamma_{ij}=\frac{1}{N}\Big\}.
	\end{align*}

    Using Lemma~\ref{lem:tau} for each term, we have
		\[
			 \W_1(\mbf{A}_{\mu}(x_i, \bdot), \mbf{A}_{\nu}(y_j, \bdot))\leq \tau(\Pi)\tau(L)\W_1(\Psi_{G(x_i,\bdot)}(\mu),\Psi_{G(y_j,\bdot)}(\nu)).
		\]Now, from Proposition~\ref{prop:psi_contract_unbounded} ($x_i$ belongs to $\supp \mu$ and $y_j$ to $\supp \nu$), we get
		\begin{align*}
		    \W_1(\Psi_{G(x_i,\bdot)}(\mu),&\Psi_{G(y_j,\bdot)}(\nu)) \\
		    &\leq \left[\sqrt{d} \sqrt{\ln{N}+\frac{1}{2e}}\|G\|_{Lip} + \|G\|_{\infty} + \sqrt{d} + 2\right](d(x_i,y_j) + \W_1(\mu,\nu)).
		\end{align*}
		Substituting this back into the above formula, we obtain
		\begin{align*}
		    &\W_1(\mu\mbf{A}_{\mu},\nu\mbf{A}_{\nu})\\
		    &\leq \min\Big\{\sum_{i,j}\gamma_{ij}\W_1(\mbf{A}_{\mu}(x_i,\bdot),\mbf{A}_{\nu}(y_i,\bdot)) \st \gamma_{i,j}\geq 0,~\sum_i\gamma_{ij}=\frac{1}{M},~\sum_j \gamma_{ij}=\frac{1}{N}\Big\}\\
		    &\leq \tau(\Pi)\tau(L) \min\Big\{\sum_{i,j}\gamma_{ij}\left[\sqrt{d} \sqrt{\ln{N}+\frac{1}{2e}}\|G\|_{Lip} + \|G\|_{\infty} + \sqrt{d} + 2\right](d(x_i,y_j) + \W_1(\mu,\nu)) \st \\
		    &\hspace{9cm} \gamma_{i,j}\geq 0,~\sum_i\gamma_{ij}=\frac{1}{M},~\sum_j \gamma_{ij}=\frac{1}{N}\Big\}\\
		    &= \tau(\Pi)\tau(L) \left[\sqrt{d} \sqrt{\ln{N}+\frac{1}{2e}}\|G\|_{Lip} + \|G\|_{\infty} + \sqrt{d} + 2 \right]  \Big(\W_1(\mu,\nu) + \\ &\hspace{5cm}\min\Big\{\sum_{i,j}\gamma_{ij}d(x_i,y_j) \st \gamma_{i,j}\geq 0,~\sum_i\gamma_{ij}=\frac{1}{M},~\sum_j \gamma_{ij}=\frac{1}{N}\Big\} \Big)\\
		    &= \tau(\Pi)\tau(L) \Big[\sqrt{d} \sqrt{\ln{N}+\frac{1}{2e}}\|G\|_{Lip} + \|G\|_{\infty} + \sqrt{d} + 2\Big]  \left(\W_1(\mu,\nu) +  \W_1(\mu,\nu)\right)\\
		    &=2\tau(\Pi)\tau(L) \Big[\sqrt{d} \sqrt{\ln{N}+\frac{1}{2e}}\|G\|_{Lip} + \|G\|_{\infty} + \sqrt{d} + 2\Big]\W_1(\mu,\nu),
		\end{align*}
		where we used in particular $\sum_{i,j} \gamma_{ij} = 1$. The inequality being valid for both $M$ and $N$, taking the $\min$ gives the result.

\end{proof}

\end{document}